\PassOptionsToPackage{table}{xcolor}
\documentclass{fairmeta}

\usepackage{amsmath,amssymb}
\usepackage{amsfonts}
\usepackage{url}
\usepackage{booktabs}
\usepackage{xspace}
\usepackage{amsthm}
\usepackage{algorithm}
\usepackage{algpseudocode}

\usepackage{amsmath,amsfonts,bm}

\def\eqref#1{equation~\ref{#1}}

\def\1{\bm{1}}

\DeclareMathAlphabet{\mathsfit}{\encodingdefault}{\sfdefault}{m}{sl}
\SetMathAlphabet{\mathsfit}{bold}{\encodingdefault}{\sfdefault}{bx}{n}

\newtheorem{theorem}{Theorem}[section]
\newtheorem{proposition}[theorem]{Proposition}

\newtheorem{remark}[theorem]{Remark}
\newtheorem{lemma}[theorem]{Lemma}
\newtheorem{assumption}[theorem]{Assumption}

\definecolor{WarmOrange}{HTML}{C9B3E6}
\definecolor{LightWarmOrange}{HTML}{E8DDF4}
\DeclareRobustCommand{\method}{\textsc{Meld}\xspace}
\hypersetup{
  pdftitle={MELD: Recursive Self-Improvement via Hierarchical Coordination for Multi-Agent Scientific Discovery},
  pdfauthor={Sizhe Tang, Zuyuan Zhang, Zeyu Fang, Tian Lan}
}

\title{\method: Recursive Self-Improvement via\\
Hierarchical Coordination for Multi-Agent\\
Scientific Discovery}

\author[1]{Sizhe Tang}
\author[1]{Zuyuan Zhang}
\author[1]{Zeyu Fang}
\author[1]{Tian Lan}
\affiliation[1]{George Washington University}
\abstract{Multi-agent scientific discovery with large language models raises
a coupled learning problem: research assignments determine both
which experiments are pursued and the experiences from which
agents adapt.
We introduce \method{}, a unified hierarchical test-time training
framework that jointly adapts 
a team-level Coordinator policy and agent-level individual Scientist policies within a single discovery process. The Coordinator generates a joint portfolio of search branches and research directives over a shared derivation tree, while Scientists develop hypotheses and execution plans implemented by a shared frozen Executor. Both levels receive alternating GRPO-style parameter updates under a shared smooth best-of-budget objective, allowing team organization and individual research behavior to co-evolve in a recursive self-improvement loop.
Our analysis establishes objective alignment for the hierarchy, characterizes conditions for coordination gains over independent parallel discovery and the first-order benefit of private adaptation, and derives a finite-budget
regret bound for adapted models.
On 22 MLGym and MLE-bench tasks, \method{} is best on 20 of them and achieves the highest
aggregate normalized scores on MLGym, using a 4B backbone
for coordination and research.
Most baselines use proprietary Scientists, and all systems
share one frozen Executor.
}
\authoremails{\email{s.tang1@gwu.edu}}
\metadata[Project Code]{\url{https://github.com/sizhe-tang/MELD-RSI}}
\date{October 2026}

\begin{document}
\maketitle

\section{Introduction}

Scientific discovery requires judiciously allocating a limited exploration/evaluation budget across competing hypotheses. Large language models (LLMs) increasingly support this process by proposing, testing, and refining candidates~\citep{funsearch,alphaevolve}, while test-time training (TTT) uses discovery feedback to adapt their research policies~\citep{thetaevolve,tttdiscover,arts}. For a team of collaborative agents, each local research assignment determines both the team's overall progress and the learning experience that shapes the team's subsequent research strategy.
Agent coordination in scientific discovery, such as focusing collective effort on promising research lineages while maintaining sufficient exploration, becomes part of the learning process rather than merely a scheduling mechanism.

Existing approaches provide complementary capabilities. TTT-Discover and ARTS adapt task-local proposal or research policies~\citep{tttdiscover,arts}, while CORAL coordinates parallel discovery through shared persistent memory~\citep{coral}. Agent Q-Mix learns communication and coordination across a training distribution~\citep{agentqmix,rashid2018qmix}, without jointly adapting its coordination strategy and underlying LLM researchers to each new discovery task. These advances leave a central joint learning problem unresolved: \emph{how can a team-level coordination policy and individual agent-level research policies be learned together within a single discovery process?} We address this problem through a unified framework that jointly adapts both levels via task-local parameter updates, allowing team organization and individual research behavior to co-evolve.

We introduce \textbf{\method{}}, a unified hierarchical TTT framework for multi-agent scientific discovery. A team-level Coordinator generates a joint portfolio of parent branches and research directives over a shared derivation tree. Initially identical Scientists, each with a private adapter, develop their assignments into hypotheses and execution plans implemented by a shared frozen Executor. The Coordinator and Scientists learn through alternating GRPO-style parameter updates~\citep{shao2024deepseekmath} under a shared smooth best-of-budget objective. Thus, \method{} creates a recursive self-improvement (RSI) loop: Joint portfolio coordination organizes research across agents and shapes the team's learning experience, while private agents' evolving capabilities and outcomes guide subsequent coordination.

Our theoretical analysis establishes objective alignment for the proposed hierarchy, characterizes conditions for coordination gains over independent parallel discovery at matched individual outcome distributions and budgets, and quantifies the first-order benefit of private adaptation at a common model. We further derive a finite-budget regret bound for adapted models and a conditional optimization rate. Across 12 MLGym~\citep{nathani2025mlgym} and 10 MLE-bench~\citep{chan2025mlebench} tasks, \method{} attains the best three-seed mean score among the compared methods on 20 tasks (Table~\ref{tab:mlgym-main-results}) and the highest aggregate normalized scores on MLGym, with an interquartile mean of $0.748$ versus $0.684$ for ARTS$^{*}$. These system-level comparisons use a 4B backbone for \method{}'s coordination and research, proprietary Scientist models for most baselines, and the same frozen Executor throughout.

Our contributions are threefold:
\begin{itemize}
    \item \textbf{A unified framework for task-local RSI.}
    We jointly learn a team-level portfolio coordination policy and agent-level Scientist research policies within a shared-objective hierarchy in a single discovery task.

    \item \textbf{Theory of coordination and private adaptation.}
   We quantify the benefits of \method{} and relate surrogate optimization to finite-budget discovery regret at adapted policy.

    \item \textbf{Strong empirical performance.}
    We report the best score on 20 of 22 MLGym and MLE-bench tasks and leading aggregate performance on MLGym with a compact research backbone, complemented by search trajectories and adaptation ablations.
\end{itemize}

\section{Related Work}

\paragraph{Test-Time Training for Discovery.}
TTT updates model parameters using a signal from the test instance or task~\citep{sun2020ttt}, whereas test-time scaling allocates additional inference computation without updating them~\citep{snell2024testtime}. Early discovery systems such as FunSearch~\citep{funsearch} and AlphaEvolve~\citep{alphaevolve} iteratively generate and select candidates with a fixed proposal model, while recent methods center adaptation on a single task-local proposal policy~\citep{thetaevolve,tttdiscover}, strategic advisor~\citep{pacevolvepp}, uncertainty-guided exploration component~\citep{ugttt}, or tree-search Scientist~\citep{arts}. \method instead organizes the experiences of multiple concurrently adapting research policies.

\paragraph{Multi-Agent Systems and Learned Orchestration.}
In discovery, CORAL~\citep{coral} and SwarmResearch~\citep{swarmresearch} coordinate parallel agents through shared memory or centralized orchestration. Parameter-trained systems~\citep{metaagentx,martimars2} typically learn across a training distribution, while TacoMAS~\citep{tacomas} co-evolves role-specific capabilities and topology at inference. Related MARL work learns emergent role representations during training~\citep{wang2020roma}, whereas \method induces persistent task-local specialization through assignment-conditioned private adapter updates. GPTSwarm~\citep{zhuge2024gptswarm}, ADAS~\citep{hu2024adas}, and the RL-trained Conductor~\citep{nielsen2025conductor} establish that orchestration can be learned. \method differs by alternating task-local weight updates to the team-level coordination policy and multiple agent-level Scientist policies within the same unseen discovery run, under one shared best-of-budget objective, so that team organization and individual research behavior can co-evolve.

\section{Preliminaries}
\label{sec:preliminaries}

\paragraph{Executable Discovery Task.}
Following \citet{arts}, we study test-time discovery on a single
previously unseen executable task
$\mathcal{P}=(\mathcal{W}_0,m,s_0,T)$, where $\mathcal{W}_0$ contains
the task specification, data, baseline artifact $c_0$, and evaluation
harness; $m$ is the task-specific metric; $s_0=s(c_0)$ is the normalized
baseline utility; and $T$ is a system-level budget shared by all agents.
Given the raw evaluator metric $q_m(c)$, we define the normalized utility
$s(c)=g_m\!\left(q_m(c)\right)\in[0,1]$, where $g_m$ is fixed for each task
and monotone in the preferred direction of $m$, so that larger $s(c)$ always
indicates a better artifact. In our implementation,
$g_m(q)=\sigma(z_m(q))$, where $z_m$ performs metric-specific
standardization and direction correction, and $\sigma$ is the logistic sigmoid.

Let $\mathcal{C}^{\mathrm{val}}_T(\Pi)$ denote the successfully evaluated
candidates produced within budget by an adaptive multi-agent policy
$\Pi$. We optimize expected best-of-budget utility, corresponding to a
simple-regret objective in fixed-budget pure exploration~\citep{bubeck2009pure},
\begin{equation}
    J_T(\Pi;\mathcal{P})
    =
    \mathbb{E}_{\Pi}\!\left[
        \max_{c\in\{c_0\}\cup\mathcal{C}^{\mathrm{val}}_T(\Pi)}
        s(c)
    \right],
    \qquad
    \mathfrak{R}_T=s^\star-J_T(\Pi;\mathcal{P}),
    \label{eq:discovery-objective}
\end{equation}
where $s^\star$ is the optimal feasible score. All parameter updates use
only interaction data from this task.

\paragraph{Shared Discovery Process.}
At round $t$, the shared history is represented by a rooted derivation
tree $\mathcal{G}_t=(V_t,E_t)$, whose root stores $(c_0,s_0)$. Together
with the remaining budget $b_t$, it defines the discovery state
$\mathcal{S}_t=(\mathcal{P},\mathcal{G}_t,b_t)$. Conditioned on
$\mathcal{S}_t$, policy $\Pi$ proposes $n_t\leq N$ concurrent expansions, where $N$ is the maximum number of active research units,
with expansion $i$ given by
\begin{equation}
    \xi_t^i=(p_t^i,h_t^i,e_t^i),
    \qquad
    c_t^i=\operatorname{Exec}(c_{p_t^i};h_t^i,e_t^i),
    \qquad
    (\ell_t^i,s_t^i)
    =
    \operatorname{Eval}_{m}(c_t^i;\mathcal{W}_0),
    \label{eq:discovery-expansion}
\end{equation}
where $p_t^i\in V_t$ selects a parent node, $h_t^i$ is a testable
hypothesis, and $e_t^i$ is an execution plan. The evaluator returns feedback
$\ell_t^i$ and utility $s_t^i$, with $s_t^i=\bot$ on failure.
The Executor remains frozen, confining parameter
adaptation to the policies that select and develop research directions.

Each successful evaluation adds
$v_t^i=(h_t^i,e_t^i,c_t^i,\ell_t^i,s_t^i)$ to the tree. Defining
$V_t^{+}=\{v_t^i:s_t^i\neq\bot\}$, we update $V_{t+1}=V_t\cup V_t^{+}$ and
$E_{t+1}=E_t\cup\{(p_t^i,v_t^i):v_t^i\in V_t^{+}\}$. Failed attempts consume
budget and provide learning signals but are not added as expandable nodes.
Section~\ref{sec:method} instantiates $\Pi$ hierarchically: the team-level Coordinator assigns $z_t^i=(p_t^i,d_t^i)$, and each agent-level Scientist $i$ produces $u_t^i=(h_t^i,e_t^i)$; their alternating adaptation forms \method's
multi-agent RSI loop.

\section{\method: Hierarchical Test-Time Training for Self-Improving Multi-Agent Discovery}
\label{sec:method}

\method couples joint research assignment with private policy adaptation. A learnable \emph{Coordinator} assigns research directives to $N$ agents (\emph{research units}), each pairing a private Scientist with the shared frozen Executor. The two policy levels, the team-level Coordinator and the agent-level Scientists, receive alternating task-local updates under a shared smooth best-of-budget objective. Assignments thus shape both current exploration and the experience available for private adaptation.

\subsection{Hierarchical Policy}
\label{sec:hierarchical-policy}

At round $t$, the coordination policy $\pi^{\mathrm{coord}}_\phi$ observes the shared discovery state $\mathcal S_t$ and generates a \emph{portfolio} $Z_t=\{z_t^i\}_{i\in\mathcal I_t}$ for active units $\mathcal I_t\subseteq\{1,\ldots,N\}$. Each assignment $z_t^i=(p_t^i,d_t^i)$ specifies a parent node $p_t^i$ from a retrieved tree shortlist and a research directive $d_t^i$: a schema-constrained natural-language statement of the research question, the parent's bottleneck, and the constraints to respect. The portfolio is emitted as one structured completion, allowing assignments to depend on one another before candidates are generated.

Given $(\mathcal S_t,z_t^i)$, Scientist $i$ samples a hypothesis and execution plan $u_t^i=(h_t^i,e_t^i)\sim\pi^{\mathrm{sci}}_{\theta_i}(\cdot\mid\mathcal S_t,z_t^i)$. The frozen Executor implements the plan as $c_t^i=\operatorname{Exec}(c_{p_t^i};h_t^i,e_t^i)$, which is evaluated as in Eq.~(\ref{eq:discovery-expansion}). Scientists sample independently conditional on the state and portfolio. With $U_t=\{u_t^i\}_{i\in\mathcal I_t}$,
\begin{equation}
\pi_\Theta(Z_t,U_t\mid\mathcal S_t)
=\pi^{\mathrm{coord}}_\phi(Z_t\mid\mathcal S_t)\prod_{i\in\mathcal I_t}\pi^{\mathrm{sci}}_{\theta_i}(u_t^i\mid\mathcal S_t,z_t^i),
\qquad \Theta=(\phi,\theta_{1:N}).
\label{eq:hierarchical-policy}
\end{equation}
The Coordinator decides \emph{where and what} to research; each Scientist decides \emph{how}. Coordination resides in the joint assignment policy, even though Scientists sample independently once assigned.

The Coordinator and Scientists share one frozen backbone, with a LoRA adapter~\citep{hu2022lora} $\phi$ for coordination and a private $\theta_i$ for each Scientist. All Scientists start identically, $\theta_{1,0}=\cdots=\theta_{N,0}$, and every adapter resets at task start; neither the backbone nor the Executor is updated. The tree shares artifacts and feedback, while private adapters retain assignment-conditioned updates without forcing identical policy changes.

\subsection{A Shared Best-of-Budget Objective}
\label{sec:shared-objective-rl}

Equation~(\ref{eq:discovery-objective}) values the best validated artifact within budget. Parent-relative gains can reward branch improvements that leave the best-so-far utility unchanged, while scoring candidates separately ignores their combination. We therefore evaluate each round against the shared discovery state, using a smooth potential whose terminal value approximates best-of-budget utility.

\paragraph{Smooth potential.}
Let $K$ be the evaluator-call budget allowed by $T$, and $V$ the derivation-tree node set in state $\mathcal S$, including the root. Define
\begin{equation}
\begin{aligned}
\Phi_{\tau,K}(\mathcal S)
&=\tau\log\!\Big(\tfrac{1}{K+1}\Big[\textstyle\sum_{v\in V}e^{s_v/\tau}+(K+1-|V|)\,e^{s_{\min}/\tau}\Big]\Big),\\
J_{\tau,K}(\Theta)&=\mathbb E_{\Pi_\Theta}\big[\Phi_{\tau,K}(\mathcal S_{\mathrm{end}})\big],
\end{aligned}
\label{eq:method-smooth-objective}
\end{equation}
where $\tau>0$ is a temperature, $\mathcal S_{\mathrm{end}}$ is the state at budget exhaustion, and failed or unused slots carry $s_{\min}=0$. Padding fixes the normalization as the tree grows: each validated candidate replaces a sentinel slot rather than changing the denominator. This log-sum-exp lies below the best utility by at most $\tau\log(K+1)$ (Proposition~\ref{prop:objective-alignment}(i)), making $J_{\tau,K}$ a shared surrogate of $J_T$ for both policy levels.

\paragraph{Potential increments as shared credit.}
A round receives $r_t^{\Phi}=\Phi_{\tau,K}(\mathcal S_{t+1})-\Phi_{\tau,K}(\mathcal S_t)$ for its collective outcome. The temperature sets how closely the potential tracks the maximum while retaining credit for near-frontier candidates. Joint assignments are thus evaluated through realized candidate utilities, not an explicit diversity bonus. The increments telescope, $\sum_t r_t^{\Phi}=\Phi_{\tau,K}(\mathcal S_{\mathrm{end}})-\Phi_{\tau,K}(\mathcal S_0)$; the exact score-function gradient~\citep{williams1992reinforce} credits each decision with the full downstream sum (Proposition~\ref{prop:objective-alignment}). Direct Monte Carlo estimation requires continuations to budget exhaustion, so \method limits adaptation to a \emph{one-round horizon}.

At snapshot $\mathcal S$, a policy block $\beta\in\{\mathrm{coord},\mathrm{sci}_1,\ldots,\mathrm{sci}_N\}$ takes action $a$ in context $x$: $a=Z$, $x=\mathcal S$ for the Coordinator, and $a=u^i$, $x=(\mathcal S,z^i)$ for Scientist $i$. Its local objective is
\begin{equation}
\begin{aligned}
J^{\mathrm{loc}}_\beta(\mathcal S)
&=\mathbb E_{a\sim\pi_\beta(\cdot\mid x)}\big[\Phi_{\tau,K}(\mathcal S\oplus X(a))-\Phi_{\tau,K}(\mathcal S)\big],\\
\nabla_\beta J^{\mathrm{loc}}_\beta(\mathcal S)
&=\mathbb E\big[r^\Phi(a)\,\nabla_\beta\log\pi_\beta(a\mid x)\big].
\end{aligned}
\label{eq:local-objective}
\end{equation}
Here $X(a)$ contains the round's validated candidates, $\mathcal S\oplus X(a)$ extends the snapshot with their nodes, and $r^{\Phi}(a)$ is the realized potential increment; the gradient identity is established in Proposition~\ref{prop:objective-alignment}. The expectation includes execution and evaluation randomness and the other policies' responses at the fixed behavior model. For a Scientist update, the full portfolio is also fixed. This supplies round-local credit rather than an unbiased full-return gradient: the Coordinator is credited through its portfolio, each Scientist through its response, and both against the same team outcome.

\subsection{Alternating Test-Time Adaptation}
\label{sec:coadaptation}

\paragraph{Matched-snapshot groups.}
Because the tree and policies evolve, relative feedback is compared at a fixed context. We combine GRPO-style sampling~\citep{shao2024deepseekmath} with a leave-one-out baseline~\citep{ahmadian2024rloo}. Each group contains $G\geq2$ independent \emph{forks} of one snapshot, with the discovery state, remaining budget, parent shortlist, and behavior model $\Theta_{\mathrm{old}}$ fixed. The $N$ research units provide parallel exploration within a fork, while the $G$ forks provide matched-context samples for adaptation. In a \emph{Coordinator group}, each fork samples a complete portfolio $Z^{(g)}\sim\pi^{\mathrm{coord}}_{\phi_{\mathrm{old}}}(\cdot\mid\mathcal S)$ and executes its units using the current Scientist policies. In a \emph{Scientist group}, one portfolio $Z$ is fixed across forks, each Scientist independently samples $u_i^{(g)}\sim\pi^{\mathrm{sci}}_{\theta_{i,\mathrm{old}}}(\cdot\mid\mathcal S,z^i)$, and their responses are executed together. Fork $g$ yields validated candidates $X^{(g)}$, reward $r_g^{\Phi}=\Phi_{\tau,K}(\mathcal S\oplus X^{(g)})-\Phi_{\tau,K}(\mathcal S)$, and advantage $\widehat A_g=r^{\Phi}_g-\frac{1}{G-1}\sum_{g'\neq g}r^{\Phi}_{g'}$, whose baseline excludes fork $g$ and, by conditional independence, preserves the sequence-score estimator's expectation in Eq.~(\ref{eq:local-objective}). Scientists share this team advantage but update only through their own response tokens, without pooling samples across snapshots or private policies.

\paragraph{Token-level update.}
For the policy being updated, let $y_g=(y_{g,1},\ldots,y_{g,L_g})$ be its structured response in fork $g$ and $\rho_{g,j}=\pi_\beta(y_{g,j}\mid x,y_{g,<j})/\pi_{\beta,\mathrm{old}}(y_{g,j}\mid x,y_{g,<j})$ the token importance ratio. Broadcasting $\widehat A_g$ to its response tokens gives the token-averaged PPO-style surrogate~\citep{schulman2017ppo},
\begin{equation}
\mathcal L_\beta
=-\frac{1}{G}\sum_{g=1}^{G}\frac{1}{L_g}\sum_{j=1}^{L_g}
\min\!\big(\rho_{g,j}\widehat A_g,\ \operatorname{clip}(\rho_{g,j},1-\epsilon_{\mathrm{lo}},1+\epsilon_{\mathrm{hi}})\,\widehat A_g\big)
+\beta_{\mathrm{KL}}\,\widehat D_{\mathrm{KL}}\big(\pi_\beta\,\|\,\pi_\beta^{\mathrm{ref}}\big),
\label{eq:grpo_objective}
\end{equation}
where $\pi_\beta^{\mathrm{ref}}$ is the frozen task-start policy, distinct from the per-group behavior policy $\pi_{\beta,\mathrm{old}}$, and $\widehat D_{\mathrm{KL}}$ is token-averaged on the same responses. Token-length normalization, clipping, and KL regularization define the implemented surrogate rather than the exact sequence-score gradient of Eq.~(\ref{eq:local-objective}) (Appendix~\ref{app:objective-proof}).

\paragraph{Update and commit.}
Algorithm~\ref{alg:matched-group} summarizes a complete group with sufficient remaining calls. Coordinator and Scientist groups alternate using fresh samples, with no replay, until the budget no longer admits a full group, after which the current policies are rolled out without further updates (Appendix~\ref{app:hparams}). Fork rewards share the pre-commit snapshot, and their sum need not equal the merged potential increment. Because updated policies collect the next group, Coordinator updates change subsequent learning assignments while Scientist updates change the outcomes informing subsequent coordination, which enables policy-level RSI.

\begin{algorithm}[t]
\caption{One complete matched-snapshot update group in \method}
\label{alg:matched-group}
\begin{algorithmic}[1]
\State Freeze $\mathcal S$ and behavior policies $\Theta_{\mathrm{old}}$ during collection.
\State Sample independent portfolios (Coordinator group) or one shared portfolio (Scientist group).
\State For each fork, sample Scientists' responses, execute and evaluate them, and compute $r_g^\Phi$.
\State Compute leave-one-out advantages $\widehat A_g$.
\State Take one step of Eq.~(\ref{eq:grpo_objective}) on $\phi$ (Coordinator) or on all private $\theta_i$ in parallel (Scientist).
\State Commit validated candidates, charge all evaluator calls to the budget, and switch the group type.
\end{algorithmic}
\end{algorithm}

\section{Theoretical Analysis}
\label{sec:theory}

This section shows that both levels of \method follow partial gradients of one objective aligned with $J_T$ (Section~\ref{sec:theory-objective}), quantifies what coordination and private adaptation add over parallel discovery (Section~\ref{sec:theory-organization}), and bounds the discovery regret after adaptation (Section~\ref{sec:theory-regret}), with proofs in Appendix~\ref{app:theory-proofs}.

\subsection{One Objective Across the Hierarchy}
\label{sec:theory-objective}

Although the team-level Coordinator and the agent-level Scientists act on different objects, Algorithm~\ref{alg:matched-group} trains both levels with the same reward $r^{\Phi}_t$. We show that the exact gradients of the two levels are partial gradients of a single objective approximating $J_T$ in Eq.~(\ref{eq:discovery-objective}).

\paragraph{Notation.}
Let $\Pi_\Theta$ denote the policy that applies $\pi_\Theta$ of Eq.~(\ref{eq:hierarchical-policy}) at every round with $\Theta$ fixed. A run of $\Pi_\Theta$ draws $\mathcal S_0\sim\nu$ and evolves by the transition kernel $\mathcal Q(\mathcal S_{t+1}\mid\mathcal S_t,Z_t,U_t)$, which comprises execution, evaluation, and bookkeeping, and $\mathbb E_{\Pi_\Theta}$ denotes expectation over runs. A policy block $\beta\in\{\mathrm{coord},\mathrm{sci}_1,\ldots,\mathrm{sci}_N\}$ acts on a context $x$ with an action $a$ (Section~\ref{sec:shared-objective-rl}), with $(a,x)=(Z_t,\mathcal S_t)$ for the Coordinator, which acts every round, and $(a,x)=(u^i_t,(\mathcal S_t,z^i_t))$ for Scientist $i$, which acts when $i\in\mathcal I_t$. We write $(a^{\beta}_t,x^{\beta}_t)$ for this pair and $\nabla_\beta$ for the gradient with respect to block $\beta$. Let $B(\mathcal S)=\max_{v\in V}s_v$, so that $J_T(\Pi_\Theta;\mathcal P)=\mathbb E_{\Pi_\Theta}[B(\mathcal S_{\mathrm{end}})]$, and let $R^{\Phi}_{t:\mathrm{end}}=\sum_{t'\ge t}r^{\Phi}_{t'}=\Phi_{\tau,K}(\mathcal S_{\mathrm{end}})-\Phi_{\tau,K}(\mathcal S_t)$ be the \emph{reward-to-go} of the team, which telescopes because $r^{\Phi}_t$ is a potential-based shaping reward~\citep{ng1999shaping}.

\begin{assumption}[Regularity]
\label{ass:regularity}
(R1) Given the sampled actions, $\nu$ and $\mathcal Q$ do not depend on $\Theta$. (R2) Each $\pi_\beta(\cdot\mid x)$ is differentiable in its block, with support independent of the block's parameters. (R3) Block gradients and expectations over runs may be interchanged.
\end{assumption}

Conditions (R1) and (R2) hold for softmax token policies with a frozen Executor and evaluator, and (R3) is the standard dominated-differentiation condition (Appendix~\ref{app:trajectory-factorization}).

\begin{proposition}[One objective across the hierarchy]
\label{prop:objective-alignment}
(i) For every $\Theta$, $J_T(\Pi_\Theta;\mathcal P)-\tau\log(K+1)\le J_{\tau,K}(\Theta)\le J_T(\Pi_\Theta;\mathcal P)$. Under Assumption~\ref{ass:regularity}: (ii) for every block $\beta$,
\begin{equation}
\nabla_\beta J_{\tau,K}(\Theta)
=\mathbb E_{\Pi_\Theta}\Big[\textstyle\sum_{t:\,\beta\text{ acts at }t}R^{\Phi}_{t:\mathrm{end}}\,\nabla_\beta\log\pi_\beta(a_t^\beta\mid x_t^\beta)\Big];
\label{eq:block-gradient}
\end{equation}
(iii) writing $R^{\Phi}_{t:\mathrm{end}}=r^{\Phi}_t+R^{\Phi}_{t+1:\mathrm{end}}$ in Eq.~(\ref{eq:block-gradient}), the $r^{\Phi}_t$-terms sum to $\mathbb E_{\Pi_\Theta}\big[\sum_{t:\,\beta\text{ acts at }t}\nabla_\beta J^{\mathrm{loc}}_\beta(\mathcal S_t)\big]$ with $J^{\mathrm{loc}}_\beta$ as in Eq.~(\ref{eq:local-objective}), evaluated for a Scientist block at the realized portfolio $Z_t$; (iv) for a matched-snapshot group at $\mathcal S$ with behavior model $\Theta_{\mathrm{old}}$ (Section~\ref{sec:coadaptation}), the leave-one-out estimator satisfies $\mathbb E\big[\widehat A_g\,\nabla_\beta\log\pi_\beta(a_g\mid x)\big]=\nabla_\beta J^{\mathrm{loc}}_\beta(\mathcal S)$ at $\Theta_{\mathrm{old}}$, with $a_g$ the action of block $\beta$ in fork $g$.
\end{proposition}

Proposition~\ref{prop:objective-alignment} establishes that the two levels of \method share a single objective, which approximates $J_T$ up to the smoothing error $\tau\log(K+1)$ (i). Because the Executor and the evaluator are frozen and $\pi_\Theta$ factorizes as in Eq.~(\ref{eq:hierarchical-policy}), the gradient of every block weights its own score function by the same reward-to-go (ii), which is the shared-return policy gradient of cooperative multi-agent learning~\citep{peshkin2000learning,foerster2018coma}. The population update behind Algorithm~\ref{alg:matched-group} is therefore block-coordinate ascent on the snapshot-local objectives $J^{\mathrm{loc}}_\beta(\mathcal S_t)$, whose gradients form the current-round component of Eq.~(\ref{eq:block-gradient}) (iii) and are estimated without bias by the one-round estimator of Section~\ref{sec:coadaptation} (iv), whereas the omitted terms involving $R^{\Phi}_{t+1:\mathrm{end}}$ capture the effect of a decision on later rounds and constitute the bias of the one-round horizon. Appendix~\ref{app:token-normalization} analyzes the implemented surrogate of Eq.~(\ref{eq:grpo_objective}) and the rejection of malformed responses. Since the credit depends only on realized utilities and increases strictly in them (Appendix~\ref{app:frontier-proof}), coordination is rewarded through outcomes alone, which Section~\ref{sec:theory-organization} quantifies.

\subsection{Coordination and Private Adaptation versus Parallel Discovery}
\label{sec:theory-organization}

Parallel discovery runs $N$ units, each choosing its assignment independently given the shared state, with frozen agents~\citep{coral} or a single adapted policy~\citep{tttdiscover}, whereas \method draws the assignments jointly and adapts each Scientist privately. We compare the two designs at equal individual quality and equal budget, so that any difference is attributable to the design rather than to stronger units or additional evaluations. Fix a snapshot $\mathcal S$ with at least $N$ evaluator calls remaining and fix the Scientist models. A \emph{portfolio law} $\pi(Z\mid\mathcal S)$ draws a full portfolio $Z=(z^1,\ldots,z^N)$, after which unit $i$ produces and evaluates its candidate independently of the other units, using only $(\mathcal S,z^i)$, and returns a utility $Y_i\in[0,1]$, with $0$ for a failed call. Let $q_i$ be the $i$th marginal of $\pi$ and $\pi^{\mathrm{par}}=\prod_iq_i$ its \emph{independent parallel} counterpart. The two laws induce the same unit marginals $H=(H_1,\ldots,H_N)$, $H_i(y)=\Pr(Y_i\le y\mid\mathcal S)$, and differ only in the dependence among assignments, which is all that coordination controls. Write $B_{\mathrm{new}}=\max\{B(\mathcal S),Y_1,\ldots,Y_N\}$, $H_\pi(y)=\Pr_\pi(\max_iY_i\le y\mid\mathcal S)$, and $\Delta\mathcal A=\sum_i(e^{Y_i/\tau}-1)$ for the soft-mass added to $\mathcal A(\mathcal S)=\sum_{v\in V}e^{s_v/\tau}+(K+1-|V|)$, the bracket in Eq.~(\ref{eq:method-smooth-objective}), so that $r^{\Phi}=\psi(\Delta\mathcal A)$ for the strictly concave and increasing function $\psi(w)=\tau\log(1+w/\mathcal A(\mathcal S))$. Let
\begin{equation}
\mathrm{MN}(H)=\min_{c\in[B(\mathcal S),1]}\Big\{c+\textstyle\sum_i\int_c^1\big(1-H_i(y)\big)\,dy\Big\}
\label{eq:coordination-objects}
\end{equation}
be the Meilijson--N\'adas bound on $\mathbb E[B_{\mathrm{new}}]$ given $H$~\citep{meilijson1979convex}, with $c\ge B(\mathcal S)$ because $B_{\mathrm{new}}$ includes the incumbent.

\begin{theorem}[Value of coordination at matched marginals]
\label{thm:coordination-advantage}
In the setting above, for every portfolio law $\pi$:
(i) $\mathbb E_\pi[B_{\mathrm{new}}]-\mathbb E_{\pi^{\mathrm{par}}}[B_{\mathrm{new}}]=\int_{B(\mathcal S)}^1\big(\prod_iH_i(y)-H_\pi(y)\big)\,dy\le\mathrm{MN}(H)-\mathbb E_{\pi^{\mathrm{par}}}[B_{\mathrm{new}}]$;
(ii) $\mathbb E_\pi[r^{\Phi}]\le\psi(\mathbb E[\Delta\mathcal A])$, where $\mathbb E[\Delta\mathcal A]$ depends on $H$ only, with equality if and only if $\Delta\mathcal A$ is $\pi$-a.s.\ constant; under $\pi^{\mathrm{par}}$ the inequality is strict unless every $Y_i$ is degenerate;
(iii) if $(Y_1,\ldots,Y_N)$ is negatively associated under $\pi$~\citep{joagdev1983negative}, then $\mathbb E_\pi[r^{\Phi}]\ge\mathbb E_{\pi^{\mathrm{par}}}[r^{\Phi}]$.
\end{theorem}

Theorem~\ref{thm:coordination-advantage} shows that an independent law can change $\mathbb E[B_{\mathrm{new}}]$ and $\mathbb E[r^{\Phi}]$ only through the marginals $H$, whereas a joint law with the same marginals adds the dependence term in (i), up to the coupling bound $\mathrm{MN}(H)$. By strict concavity of $\psi$, the training reward at fixed marginals is largest when $\Delta\mathcal A$ is deterministic, which no independent law with non-degenerate units achieves (ii), and it is at least the parallel reward under negative association (iii). When one of $N$ directions has utility $s_{\mathrm{hi}}>B(\mathcal S)$ and the others $s_{\mathrm{lo}}\le B(\mathcal S)$, assigning them by a uniformly random permutation attains the bounds in (i) and (ii) and exceeds independent assignment in expected best-of-round utility by $(s_{\mathrm{hi}}-B(\mathcal S))(1-1/N)^N$ (Appendix~\ref{app:coordination-construction}). Since the Coordinator emits the portfolio as one completion (Section~\ref{sec:hierarchical-policy}) and is trained on $r^{\Phi}$, such laws are expressible and rewarded by its objective.

\paragraph{Private adaptation.}
To isolate the learning signal that coordination gives each Scientist, fix $\mathcal S$ and a full portfolio $Z$, let all Scientists start from a common model with weights $W_0$, and let $\bar r^{\Phi}(W_1,\ldots,W_N)=\mathbb E[r^{\Phi}\mid\mathcal S,Z]$ be the expected round reward when Scientist $i$ responds from weights $W_i=W_0+\gamma U_iV_i$, with LoRA scaling $\gamma>0$ and factors $\theta_i=(U_i,V_i)$ starting from the standard zero-output initialization $(0,V_0)$~\citep{hu2022lora}. Write $\Gamma_i=(\nabla_{W_i}\bar r^{\Phi})V_0^{\top}$ at the common model, $\bar\Gamma=N^{-1}\sum_i\Gamma_i$, and $\mathcal D(Z;\mathcal S)=N^{-2}\sum_{i<j}\|\Gamma_i-\Gamma_j\|_F^2$ for the dispersion of these assignment-conditioned gradients. We compare one population gradient step of size $\eta$ on each private adapter with the same step on a single shared adapter, whose population gradient is the average of the $N$ Scientist gradients because Eq.~(\ref{eq:grpo_objective}) averages over the $GN$ pooled responses.

\begin{proposition}[Private adaptation retains dispersion]
\label{prop:private-adaptation}
For differentiable $\bar r^{\Phi}$, the first-order change of $\bar r^{\Phi}$ divided by $N$ is $\eta\gamma^2\big(\|\bar\Gamma\|_F^2+\mathcal D(Z;\mathcal S)\big)$ for private adapters and $\eta\gamma^2\|\bar\Gamma\|_F^2$ for the shared adapter. If every Scientist receives the same assignment, then $\mathcal D(Z;\mathcal S)=0$.
\end{proposition}

Proposition~\ref{prop:private-adaptation} shows that a shared adapter follows only the mean signal $\bar\Gamma$, whereas private adapters also retain the dispersion $\mathcal D(Z;\mathcal S)$. Because the Coordinator sets the law of $Z$, it also sets the expected dispersion available to the Scientists, which is the mechanism behind the co-adaptation loop of Section~\ref{sec:coadaptation}. Pooling forgoes exactly this term, since a shared adapter cannot move along $\Gamma_i-\bar\Gamma$, in exchange for the lower estimation variance of averaging $GN$ responses.

\subsection{From Co-Adaptation to Discovery Regret}
\label{sec:theory-regret}

We now show that adaptation converges, in the sense that the discovery regret in Eq.~(\ref{eq:discovery-objective}) of the model after $C$ Coordinator--Scientist cycles decays geometrically in $C$ to a floor. Let $\Omega$ be a convex set that almost surely contains every iterate of Algorithm~\ref{alg:matched-group} from the initial model $\Theta^{(0)}$, $\mathcal H=\{\Pi_\Theta:\Theta\in\Omega\}$, $F(\Theta)=J_{\tau,K}(\Theta)$ with $F^{\star}=\sup_{\Omega}F$, and $\varepsilon_{\mathrm{cov}}=s^{\star}-\sup_{\Pi\in\mathcal H}J_T(\Pi;\mathcal P)$ the coverage error of $\mathcal H$. Unlike a run of $\Pi_\Theta$, Algorithm~\ref{alg:matched-group} updates $\Theta$ between groups, performing $M=2C$ block updates $\Theta^{(m+1)}=\Theta^{(m)}+\eta\,\widehat\nabla_m$ that alternate between the Coordinator block $\phi$ and the stacked Scientist block $\theta_{1:N}$, where $\widehat\nabla_m$ is the negative gradient of Eq.~(\ref{eq:grpo_objective}) with respect to the block $\mathsf b_m$ updated at step $m$, computed on the group collected at $\Theta^{(m)}$ and padded with zeros outside $\mathsf b_m$, and $\mathcal F_m$ is the history before the group at step $m$ is collected.

\begin{assumption}[Curvature and estimation]
\label{ass:regret}
(A1) $\nabla F$ is $L$-Lipschitz on $\Omega$ and $\tfrac12\|\nabla F(\Theta)\|^2\ge\mu\,(F^{\star}-F(\Theta))$ for all $\Theta\in\Omega$, with $\mu\le L$. (A2) With $\bar\nabla_m=\mathbb E[\widehat\nabla_m\mid\mathcal F_m]$, each block estimate satisfies $\|\bar\nabla_m-\nabla_{\mathsf b_m}F(\Theta^{(m)})\|\le\delta$ and $\mathbb E[\|\widehat\nabla_m-\bar\nabla_m\|^2\mid\mathcal F_m]\le\sigma^2$.
\end{assumption}

\begin{theorem}[Convergence of the discovery regret]
\label{thm:regret}
Under Assumptions~\ref{ass:regularity} and~\ref{ass:regret}, for $0<\eta\le1/(4L)$ and $q=1-\eta\mu/4$,
\begin{equation}
\mathfrak R_T\big(\Pi_{\Theta^{(M)}}\big)=s^{\star}-\mathbb E\big[J_T(\Pi_{\Theta^{(M)}};\mathcal P)\big]
\le\varepsilon_{\mathrm{cov}}+\tau\log(K+1)+q^{C}\big(F^{\star}-F(\Theta^{(0)})\big)+\frac{5L\eta\sigma^2+8\delta^2}{\mu}.
\label{eq:regret-bound}
\end{equation}
\end{theorem}

Theorem~\ref{thm:regret} shows that the surrogate optimality gap $F^{\star}-F(\Theta^{(0)})$ decays geometrically in the number of Coordinator--Scientist cycles, with two block updates per cycle regardless of $N$ because the Scientists update in parallel as one block, down to a floor set by the variance $\sigma^2$ of the group estimate, which averages $G$ forks, and by its bias $\delta$, which absorbs the one-round horizon, the implemented surrogate, and the mismatch between the snapshot of a group and the run objective (Appendix~\ref{app:block-rate}). Adaptation does not reduce the coverage and smoothing errors, and Eq.~(\ref{eq:regret-bound}) bounds the regret of a full-budget rollout of the adapted model, which Appendix~\ref{app:regret-any-state} extends, under a snapshot-mismatch assumption, to the interleaved \method run, in which adaptation consumes part of the budget before the final model is rolled out.

\section{Experiments}
\label{sec:experiments}

\subsection{Experimental Setup}
\label{sec:experiment-setup}

\paragraph{Tasks and metrics.}
We use 12 tasks from MLGym~\citep{nathani2025mlgym} and 10 from MLE-bench~\citep{chan2025mlebench}
with private evaluators and a shared nominal budget of scored
candidates per task (Appendix~\ref{app:experimental-details}). Table~\ref{tab:mlgym-main-results}
reports the mean best held-out score over three seeds, and we aggregate the 12 normalized MLGym
scores (Appendix~\ref{app:aggregate}) by Median, interquartile mean (IQM), Mean, and Optimality
Gap~\citep{agarwal2021rliable}.

\paragraph{Compared systems.}
We rerun AIRA~\citep{toledo2026ai}, MLEvolve~\citep{du2026mlevolve}, and CORAL~\citep{coral} in
the same harness with frozen Claude Haiku~4.5 as their Scientist and as the shared Executor
(Appendix~\ref{app:baselines}). \method adapts LoRA adapters on a frozen
Qwen3-4B-Instruct-2507~\citep{yang2025qwen3} backbone (Appendix~\ref{app:hparams}), so this
system-level comparison tests whether task-time adaptation lets a compact open-weight Scientist
compete with frozen proprietary Scientists.
We also rerun ARTS$^{*}$~\citep{arts}, the test-time-training variant of ARTS, which fine-tunes a
single tree-search Scientist on the same backbone.

\subsection{Main Results}
\label{sec:task-level-results}

\newcommand{\meldbestcell}[1]{\cellcolor{WarmOrange}\textbf{#1}}
\newcommand{\meldsecondcell}[1]{\cellcolor{LightWarmOrange}\textbf{#1}}

\begin{table}[t]
    \centering
    \caption{Best held-out scores within the nominal evaluator-call budget
    on 12 MLGym and 10 MLE-bench tasks, averaged over three independent
    seeds. All baseline values are our same-harness reruns with the shared
    Executor and evaluator. Baseline is the starting solution of each task,
    and Human Best is our independently measured task-specific reference for MLGym
    (Appendix~\ref{app:tasks}) and the Kaggle top-1 score for MLE-bench.
    Dark and light purple mark the best and second-best automated methods.}
    \label{tab:mlgym-main-results}
    \resizebox{\textwidth}{!}{%
    \begin{tabular}{@{}lccccccc@{}}
        \toprule
        & & \multicolumn{4}{c}{\textit{Prior Works}} & \multicolumn{1}{c}{\textit{Ours}} & \\
        \cmidrule(lr){3-6}\cmidrule(lr){7-7}
        \textbf{Task} & \textbf{Baseline} & \textbf{CORAL} & \textbf{AIRA} & \textbf{MLEvolve} & \textbf{ARTS$^{*}$} & \textbf{\method} & \textbf{Human Best} \\
        \specialrule{\lightrulewidth}{\aboverulesep}{0pt}
        \rowcolor[HTML]{8C6BB1}\multicolumn{8}{c}{\rule[-1.1ex]{0pt}{3.6ex}\color{white}\textbf{MLGym}} \\
        \specialrule{\lightrulewidth}{0pt}{\belowrulesep}
        Titanic (acc $\uparrow$) & 0.766 & 0.947 & 0.921 & 0.947 & \meldsecondcell{0.952} & \meldbestcell{0.972} & 0.830 \\
        CIFAR-10 (acc $\uparrow$) & 0.497 & \meldsecondcell{0.962} & \meldbestcell{0.963} & 0.955 & 0.959 & 0.960 & 0.994 \\
        Fashion-MNIST (acc $\uparrow$) & 0.848 & 0.943 & \meldsecondcell{0.947} & 0.937 & 0.945 & \meldbestcell{0.957} & 0.968 \\
        House Prices ($R^2$ $\uparrow$) & 0.880 & 0.924 & 0.929 & 0.929 & \meldsecondcell{0.931} & \meldbestcell{0.943} & 0.990 \\
        \midrule
        MNLI (acc $\uparrow$) & 52.51 & \meldsecondcell{84.22} & 84.19 & 84.00 & 84.13 & \meldbestcell{84.23} & 92.50 \\
        Language Modeling (loss $\downarrow$) & 4.238 & \meldsecondcell{3.665} & 3.682 & 3.766 & 3.672 & \meldbestcell{3.663} & 3.280 \\
        \midrule
        Battle of the Sexes (payoff $\uparrow$) & 1.023 & \meldsecondcell{1.446} & 1.445 & 1.445 & \meldsecondcell{1.446} & \meldbestcell{1.447} & 1.667 \\
        Prisoner's Dilemma (payoff $\uparrow$) & 2.372 & 2.637 & 2.636 & \meldsecondcell{2.644} & 2.642 & \meldbestcell{2.646} & 3.000 \\
        Colonel Blotto (score $\uparrow$) & $-0.248$ & \meldsecondcell{0.253} & \meldsecondcell{0.253} & 0.252 & 0.252 & \meldbestcell{0.254} & 0.500 \\
        \midrule
        MountainCar (reward $\uparrow$) & 53.26 & 81.91 & 71.16 & 67.62 & \meldsecondcell{84.53} & \meldbestcell{93.51} & 100.00 \\
        Breakout (reward $\uparrow$) & 41.93 & 71.64 & 73.94 & 50.20 & \meldsecondcell{78.46} & \meldbestcell{84.97} & 100.00 \\
        MetaMaze (reward $\uparrow$) & 20.12 & 74.88 & 72.39 & 58.59 & \meldsecondcell{75.41} & \meldbestcell{81.19} & 100.00 \\
        \specialrule{\lightrulewidth}{\aboverulesep}{0pt}
        \rowcolor[HTML]{8C6BB1}\multicolumn{8}{c}{\rule[-1.1ex]{0pt}{3.6ex}\color{white}\textbf{MLE-bench}} \\
        \specialrule{\lightrulewidth}{0pt}{\belowrulesep}
        Spaceship Titanic (acc $\uparrow$) & 0.7747 & \meldbestcell{0.8356} & 0.8287 & \meldbestcell{0.8356} & \meldsecondcell{0.8352} & 0.8310 & 0.828 \\
        Nomad 2018 (MCW-RMSLE $\downarrow$) & 0.1086 & 0.0543 & 0.0599 & 0.0587 & \meldsecondcell{0.0535} & \meldbestcell{0.0501} & 0.051 \\
        Jigsaw Toxic (CW-AUC $\uparrow$) & 0.9712 & 0.9789 & 0.9797 & \meldsecondcell{0.9803} & 0.9801 & \meldbestcell{0.9805} & 0.989 \\
        \midrule
        APTOS 2019 (QWK $\uparrow$) & 0.7287 & 0.8407 & 0.8624 & 0.7811 & \meldsecondcell{0.8625} & \meldbestcell{0.8661} & 0.936 \\
        Plant Pathology (MCW-AUC $\uparrow$) & 0.9481 & 0.9527 & 0.9851 & 0.9915 & \meldsecondcell{0.9921} & \meldbestcell{0.9957} & 0.984 \\
        Histopathologic Cancer (AUC $\uparrow$) & 0.9463 & 0.9906 & 0.9874 & \meldsecondcell{0.9914} & 0.9908 & \meldbestcell{0.9961} & 1.000 \\
        \midrule
        Vesuvius Ink ($F_{0.5}$ $\uparrow$) & 0.3293 & \meldsecondcell{0.3948} & 0.3817 & 0.3701 & 0.3942 & \meldbestcell{0.4084} & 0.831 \\
        CHAMPS (log-MAE $\downarrow$) & 1.2353 & 1.2327 & 1.2256 & 1.2263 & \meldsecondcell{1.2247} & \meldbestcell{1.1592} & $-3.240$ \\
        HMS Brain (KL $\downarrow$) & 0.8865 & 0.8270 & 0.7389 & \meldsecondcell{0.7286} & 0.7352 & \meldbestcell{0.5404} & 0.272 \\
        RSNA Brain Tumor (AUC $\uparrow$) & 0.5529 & 0.6647 & 0.6471 & 0.6482 & \meldsecondcell{0.6709} & \meldbestcell{0.7129} & 0.621 \\
        \bottomrule
    \end{tabular}%
    }
\end{table}

\paragraph{Broad gains across task families.}
\method achieves the best reported point estimate on 20 of 22 tasks, trailing only on CIFAR-10
(AIRA, by $0.003$) and Spaceship Titanic (CORAL and MLEvolve, by $0.0046$;
Table~\ref{tab:mlgym-main-results}). On MLGym, gains are generally modest on supervised-learning
and game-theory tasks but reach $8.98$, $6.51$, and $5.78$ points over ARTS$^{*}$, the strongest
baseline on MountainCar, Breakout, and MetaMaze, so the advantage spans task families but is
largest on sequential control.

\begin{figure}[t]
    \centering
    \begin{subfigure}[t]{0.245\textwidth}
        \centering
        \includegraphics[width=\linewidth]{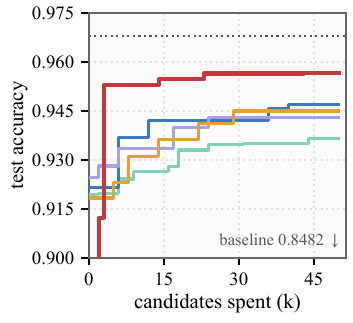}
        \caption{Fashion-MNIST}
        \label{fig:dyn-fmnist}
    \end{subfigure}\hfill
    \begin{subfigure}[t]{0.245\textwidth}
        \centering
        \includegraphics[width=\linewidth]{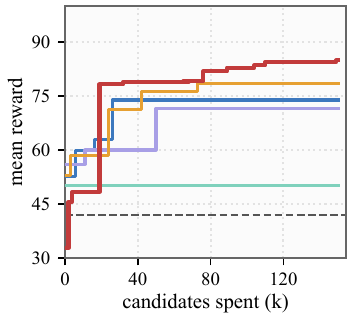}
        \caption{Breakout}
        \label{fig:dyn-breakout}
    \end{subfigure}\hfill
    \begin{subfigure}[t]{0.245\textwidth}
        \centering
        \includegraphics[width=\linewidth]{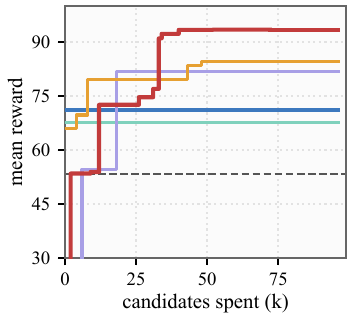}
        \caption{MountainCar}
        \label{fig:dyn-mountaincar}
    \end{subfigure}\hfill
    \begin{subfigure}[t]{0.245\textwidth}
        \centering
        \includegraphics[width=\linewidth]{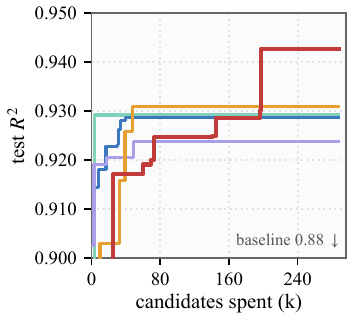}
        \caption{House Prices}
        \label{fig:dyn-houseprice}
    \end{subfigure}

    \caption{Best-so-far held-out score versus scored candidates for
    \method{} (red), ARTS$^{*}$ (orange), AIRA (blue), CORAL (purple), MLEvolve (green),
    the Baseline (dashed), and the Human Best (dotted, when visible).}
    \label{fig:training-dynamics}
\end{figure}

\paragraph{Early and delayed improvements.}
In Figure~\ref{fig:training-dynamics}, \method leads early on Fashion-MNIST, whereas on Breakout
and MountainCar it overtakes the strongest early baselines later in the search. On Breakout and
House Prices, it keeps improving long after every baseline has plateaued, overtaking ARTS$^{*}$ on
House Prices only after two-thirds of the budget. Its gains are therefore not confined to the
initial batch. These are search trajectories; the adaptation ablations below compare update
configurations.

\begin{figure}[t]
    \centering
    \includegraphics[width=\textwidth]{
        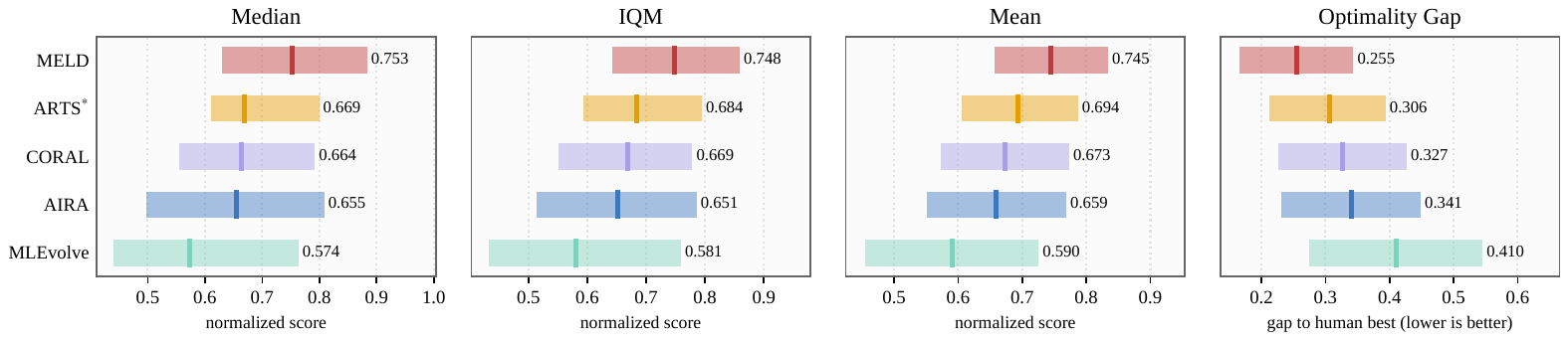}
    \caption{Median, IQM, Mean, and Optimality Gap of normalized scores (Baseline${}=0$, Human
    Best${}=1$) over the 12 MLGym tasks, with 95\% bootstrap intervals. \method is best on all four.}
    \label{fig:aggregate}
\end{figure}

\paragraph{Aggregate performance and scope.}
All four aggregate statistics rank \method first
(Figure~\ref{fig:aggregate}). Its Median, IQM, and Mean are $0.753$,
$0.748$, and $0.745$, versus $0.669$, $0.684$, and $0.694$ for ARTS$^{*}$,
the strongest aggregate baseline. Optimality Gap decreases from $0.306$
to $0.255$ (about $17\%$). Median and IQM gains indicate that the advantage
is not confined to the largest control-task improvements.

\subsection{Ablations}
\label{sec:adaptation_ablations}

We test whether adapting both policy levels improves candidate generation beyond either update
alone. On Iterated Prisoner's Dilemma (PD) and Colonel Blotto (CB), full \method{}, Coordinator-only,
Scientist-only, and Frozen each generate a fresh 24-candidate batch after one Coordinator--Scientist
cycle (Appendix~\ref{app:ablation_protocol}). Over three seeds, full \method{} raises the PD batch
median from $2.35$ to $2.57$ and threshold hits from $2$ to $10$ per 24 attempts
(Table~\ref{tab:adaptation_ablation}), versus only $5$ and $6$ hits for Coordinator-only and
Scientist-only. On CB, medians stay near zero, while full \method{} raises hits from $2$ to $6$,
versus $3$, $4$, and $2$ for Coordinator-only, Scientist-only, and Frozen. Joint adaptation thus
yields stronger post-cycle generation than either update alone in these experiments, improving
median and hits on PD but only hits on CB, a gain that neither the batch median nor a cumulative
best score reveals. With one research unit ($N{=}1$) at the same budget, \method{} loses nearly half
of its CHAMPS margin over ARTS$^{*}$, as log-MAE rises from $1.1592$ to $1.1899$, so the full gain
requires multiple units.
\section{Conclusion}
\label{sec:conclusion}

We introduced \method{}, a hierarchical TTT framework that jointly adapts team-level Coordinator and agent-level Scientist policies on one unseen task. Updated policies generate the experience and feedback for later adaptation, forming a recursive self-improvement loop with team-level coordination and agent-level research exploration. Our analysis establishes objective alignment, characterizes conditions for coordination gains beyond independent parallel discovery and the first-order benefit of private adaptation, and bounds the discovery regret of adapted models. Across 12 MLGym and 10 MLE-bench tasks, \method{} is best on 20 and achieves the highest aggregate normalized scores on MLGym with a 4B backbone. Current limitations are round-local credit and task-local adaptation within a fixed research hierarchy, and cross-task transfer could be natural future work.

\bibliographystyle{assets/plainnat}
\bibliography{ref}

\clearpage
\beginappendix

\section{Proofs for Section~\ref{sec:theory}}
\label{app:theory-proofs}

This appendix contains complete proofs of Proposition~\ref{prop:objective-alignment}, Theorem~\ref{thm:coordination-advantage}, Proposition~\ref{prop:private-adaptation}, and Theorem~\ref{thm:regret}, including a two-block ascent rate with biased estimates (Appendix~\ref{app:block-rate}), and the exact-gap and frontier-seeking properties of the potential (Appendix~\ref{app:frontier-proof}). We first fix notation and collect elementary facts about the potential that are used throughout.

\subsection{Notation and Elementary Properties of the Potential}
\label{app:notation}

\paragraph{Matrices and blocks.}
For matrices $A,B$ of equal shape, $\langle A,B\rangle_F=\operatorname{tr}(A^{\top}B)$ and $\|A\|_F^2=\langle A,A\rangle_F$. When a parameter block consists of several matrices, inner products and squared norms are summed over the matrices, and every identity below holds matrix by matrix. The trainable blocks are $\beta\in\{\mathrm{coord},\mathrm{sci}_1,\ldots,\mathrm{sci}_N\}$ with policies $\pi_\beta$, actions $a^{\beta}$, and contexts $x^{\beta}$ as in Section~\ref{sec:shared-objective-rl}: $a^{\mathrm{coord}}_t=Z_t$, $x^{\mathrm{coord}}_t=\mathcal S_t$, and $a^{\mathrm{sci}_i}_t=u_t^i$, $x^{\mathrm{sci}_i}_t=(\mathcal S_t,z_t^i)$.

\paragraph{States and the potential.}
A state $\mathcal S$ carries the derivation tree with node set $V(\mathcal S)$ and utilities $s_v\in[0,1]$; the root carries $s_0$, every other node is a validated candidate, and $B(\mathcal S)=\max_{v\in V(\mathcal S)}s_v$. Since every validated candidate consumed one evaluator call, $n_{\mathrm{val}}(\mathcal S):=|V(\mathcal S)|-1\le K$, and the number of padded slots
\begin{equation}
n_{\mathrm{pad}}(\mathcal S):=K+1-|V(\mathcal S)|=K-n_{\mathrm{val}}(\mathcal S)\ \ge\ 0
\label{eq:app-padded-slots}
\end{equation}
counts evaluator calls that failed or have not been used. Failed calls therefore do not appear in $V(\mathcal S)$ but are accounted for by Eq.~(\ref{eq:app-padded-slots}). With $s_{\min}=0$, define
\begin{equation}
\mathcal A(\mathcal S)=\sum_{v\in V(\mathcal S)}e^{s_v/\tau}+n_{\mathrm{pad}}(\mathcal S)\,e^{s_{\min}/\tau},
\qquad
\Phi_{\tau,K}(\mathcal S)=\tau\log\frac{\mathcal A(\mathcal S)}{K+1},
\label{eq:app-potential}
\end{equation}
which is Eq.~(\ref{eq:method-smooth-objective}); $\mathcal A(\mathcal S)$ is a sum of exactly $K+1$ positive terms, one per evaluator slot including the root. For a state $\mathcal S'$ and a collection $X$ of new validated candidates, $\mathcal S'\oplus X$ denotes the state extended by their call-indexed nodes; repeated utilities remain separate entries. Failed calls leave $n_{\mathrm{pad}}$ unchanged because it counts failed and unused slots together.

\begin{lemma}[Slot bookkeeping]
\label{lem:app-bookkeeping}
Let $\mathcal S'$ be a state and $X$ the validated candidates produced within its remaining evaluator-call budget. Then
\begin{equation}
\mathcal A(\mathcal S'\oplus X)=\mathcal A(\mathcal S')+\sum_{c\in X}\big(e^{s(c)/\tau}-e^{s_{\min}/\tau}\big)\ \ge\ \mathcal A(\mathcal S').
\label{eq:app-bookkeeping}
\end{equation}
In particular, failed calls leave $\mathcal A$ unchanged, and for a single candidate $c$ with utility $s$,
\begin{equation}
\Phi_{\tau,K}(\mathcal S'\oplus\{c\})-\Phi_{\tau,K}(\mathcal S')=\tau\log\Big(1+\frac{e^{s/\tau}-e^{s_{\min}/\tau}}{\mathcal A(\mathcal S')}\Big).
\label{eq:app-marginal-gain}
\end{equation}
\end{lemma}

\begin{proof}
Each validated candidate $c\in X$ adds the node term $e^{s(c)/\tau}$ to the first sum in Eq.~(\ref{eq:app-potential}) and, having consumed one evaluator call, reduces $n_{\mathrm{pad}}$ by one, removing one term $e^{s_{\min}/\tau}$; a failed call reduces neither $|V|$ nor the number of remaining slots as counted in Eq.~(\ref{eq:app-padded-slots}), since $n_{\mathrm{pad}}$ counts failed and unused slots alike. This gives the equality in Eq.~(\ref{eq:app-bookkeeping}). The inequality follows from $s(c)\ge 0=s_{\min}$. Dividing Eq.~(\ref{eq:app-bookkeeping}) with $X=\{c\}$ by $\mathcal A(\mathcal S')$ and taking $\tau\log$ gives Eq.~(\ref{eq:app-marginal-gain}).
\end{proof}

\begin{lemma}[Monotone, telescoping increments]
\label{lem:app-telescoping}
Along any run, the round rewards $r_t^{\Phi}=\Phi_{\tau,K}(\mathcal S_{t+1})-\Phi_{\tau,K}(\mathcal S_t)$ satisfy $r_t^{\Phi}\ge 0$, and for any $0\le t_1\le t_2\le R$,
\begin{equation}
\sum_{t'=t_1}^{t_2-1}r_{t'}^{\Phi}=\Phi_{\tau,K}(\mathcal S_{t_2})-\Phi_{\tau,K}(\mathcal S_{t_1});
\qquad\text{in particular}\qquad
\sum_{t'=0}^{R-1}r_{t'}^{\Phi}=\Phi_{\tau,K}(\mathcal S_{\mathrm{end}})-\Phi_{\tau,K}(\mathcal S_0).
\label{eq:app-telescoping}
\end{equation}
Moreover, $R^{\Phi}_{t:\mathrm{end}}=r^{\Phi}_t+R^{\Phi}_{t+1:\mathrm{end}}$ for every round $t$.
\end{lemma}

\begin{proof}
Nonnegativity follows from Lemma~\ref{lem:app-bookkeeping} applied to the validated candidates of round $t$, since $\log$ is increasing. The sums telescope term by term. The last identity is the case $t_1=t$, $t_2=t+1$ combined with the definition $R^{\Phi}_{t:\mathrm{end}}=\Phi_{\tau,K}(\mathcal S_{\mathrm{end}})-\Phi_{\tau,K}(\mathcal S_t)$.
\end{proof}

\subsection{Trajectory Factorization and Regularity}
\label{app:trajectory-factorization}

A run of $\Pi_\Theta$ with $\Theta=(\phi,\theta_{1:N})$ held fixed produces the trajectory
\begin{equation}
\zeta=(\mathcal S_0,Z_0,U_0,\mathcal S_1,Z_1,U_1,\ldots,Z_{R-1},U_{R-1},\mathcal S_R),
\qquad \mathcal S_R=\mathcal S_{\mathrm{end}},
\label{eq:app-trajectory}
\end{equation}
where $R\le K$ is the number of nonempty rounds until the budget is exhausted, $Z_t=\{z_t^i\}_{i\in\mathcal I_t}$, and $U_t=\{u_t^i\}_{i\in\mathcal I_t}$. Let $\mathcal Q(\mathcal S_{t+1}\mid\mathcal S_t,Z_t,U_t)$ be the kernel comprising execution by the frozen Executor, evaluation, the tree update of Section~\ref{sec:preliminaries}, and budget accounting, and let $\nu$ be the law of the initial state. By the factorization of Eq.~(\ref{eq:hierarchical-policy}),
\begin{equation}
p_\Theta(\zeta)=\nu(\mathcal S_0)\prod_{t=0}^{R-1}\Big[\pi^{\mathrm{coord}}_\phi(Z_t\mid\mathcal S_t)\prod_{i\in\mathcal I_t}\pi^{\mathrm{sci}}_{\theta_i}(u_t^i\mid\mathcal S_t,z_t^i)\;\mathcal Q(\mathcal S_{t+1}\mid\mathcal S_t,Z_t,U_t)\Big].
\label{eq:app-factorization}
\end{equation}

Throughout we work under Assumption~\ref{ass:regularity}, whose conditions (R1)--(R3) refer to the factorization in Eq.~(\ref{eq:app-factorization}): (R1) $\nu$ and $\mathcal Q$ do not depend on $\Theta$; (R2) each policy density is differentiable in its own block with parameter-independent support; (R3) differentiation with respect to any block may be interchanged with the expectation over $\zeta$.

\begin{remark}[The assumptions for LLM policies]
\label{rem:app-llm-regularity}
At a fixed context, bounded-length token actions over a finite vocabulary have finite support. Softmax probabilities on a parameter-independent admissible support are differentiable, giving (R2); any decoding mask is held fixed when taking derivatives. For the full trajectory law, however, contexts depend on random execution feedback, so finite action spaces alone do not imply (R3). A sufficient condition is an integrable envelope $D_\beta(\zeta)$ such that $\sup_{\Theta\in\mathcal U}\|\nabla_\beta p_\Theta(\zeta)\|\le D_\beta(\zeta)$ in a neighborhood $\mathcal U$ of the differentiated model. The bounded potential then permits dominated differentiation. We impose (R3) on the trajectory law and its fixed-snapshot conditional laws. Assumption (R1) records the parameter independence of the Executor, evaluator, and state transition given actions.
\end{remark}

We define the filtrations under which actions are drawn. Let $\mathcal E_t=\sigma(\mathcal S_0,Z_0,U_0,\ldots,\mathcal S_t)$ be the history up to and including $\mathcal S_t$, and $\mathcal E_t^{+}=\mathcal E_t\vee\sigma(Z_t)$ the history after the portfolio of round $t$ is drawn. For $\beta=\mathrm{coord}$ write $\mathcal F_t^{\beta}=\mathcal E_t$; for $\beta=\mathrm{sci}_i$ write $\mathcal F_t^{\beta}=\mathcal E_t^{+}$. In both cases $x_t^\beta$ is $\mathcal F_t^{\beta}$-measurable, and conditionally on $\mathcal F_t^{\beta}$ the action is drawn as $a_t^\beta\sim\pi_\beta(\cdot\mid x_t^\beta)$; for $\beta=\mathrm{sci}_i$ this uses that, given $\mathcal S_t$ and $Z_t$, the responses of different Scientists are conditionally independent by Eq.~(\ref{eq:hierarchical-policy}).

\begin{lemma}[Score identity]
\label{lem:app-score-identity}
Under (R2), for every block $\beta$ and every context $x$,
\begin{equation}
\mathbb E_{a\sim\pi_\beta(\cdot\mid x)}\big[\nabla_\beta\log\pi_\beta(a\mid x)\big]=0 .
\label{eq:app-score-identity}
\end{equation}
\end{lemma}

\begin{proof}
Let $\mathcal X$ be the parameter-independent support of $\pi_\beta(\cdot\mid x)$, which is finite for bounded-length token actions (Remark~\ref{rem:app-llm-regularity}). Then
$\mathbb E[\nabla_\beta\log\pi_\beta(a\mid x)]=\sum_{a\in\mathcal X}\pi_\beta(a\mid x)\,\frac{\nabla_\beta\pi_\beta(a\mid x)}{\pi_\beta(a\mid x)}=\sum_{a\in\mathcal X}\nabla_\beta\pi_\beta(a\mid x)=\nabla_\beta\sum_{a\in\mathcal X}\pi_\beta(a\mid x)=\nabla_\beta 1=0$,
where the interchange of $\nabla_\beta$ and the finite sum uses only (R2), since $\mathcal X$ does not move with the parameters.
\end{proof}

\begin{lemma}[Baseline lemma]
\label{lem:app-baseline}
Let $\beta$ act at round $t$, and let $\chi$ be an integrable, $\mathcal F_t^{\beta}$-measurable random variable such that $\chi\,\nabla_\beta\log\pi_\beta(a_t^\beta\mid x_t^\beta)$ is integrable. Then
$\mathbb E\big[\chi\,\nabla_\beta\log\pi_\beta(a_t^\beta\mid x_t^\beta)\big]=0$.
\end{lemma}

\begin{proof}
By the tower property and $\mathcal F_t^{\beta}$-measurability of $\chi$ and of $x_t^\beta$,
$\mathbb E[\chi\,\nabla_\beta\log\pi_\beta(a_t^\beta\mid x_t^\beta)]=\mathbb E\big[\chi\,\mathbb E[\nabla_\beta\log\pi_\beta(a_t^\beta\mid x_t^\beta)\mid\mathcal F_t^{\beta}]\big]$, and the inner conditional expectation is $\mathbb E_{a\sim\pi_\beta(\cdot\mid x_t^\beta)}[\nabla_\beta\log\pi_\beta(a\mid x_t^\beta)]=0$ by Lemma~\ref{lem:app-score-identity}.
\end{proof}

\begin{lemma}[Block score decomposition]
\label{lem:app-score-decomposition}
Under (R1)--(R2), for every block $\beta$ and every trajectory $\zeta$ in the support of $p_\Theta$,
\begin{equation}
\nabla_\beta\log p_\Theta(\zeta)=\sum_{t:\,\beta\text{ acts at }t}\nabla_\beta\log\pi_\beta(a_t^\beta\mid x_t^\beta),
\label{eq:app-score-decomposition}
\end{equation}
where $\mathrm{coord}$ acts at every round and $\mathrm{sci}_i$ acts at the rounds with $i\in\mathcal I_t$.
\end{lemma}

\begin{proof}
Taking logarithms in Eq.~(\ref{eq:app-factorization}),
$\log p_\Theta(\zeta)=\log\nu(\mathcal S_0)+\sum_t\big[\log\pi^{\mathrm{coord}}_\phi(Z_t\mid\mathcal S_t)+\sum_{i\in\mathcal I_t}\log\pi^{\mathrm{sci}}_{\theta_i}(u_t^i\mid\mathcal S_t,z_t^i)+\log\mathcal Q(\mathcal S_{t+1}\mid\mathcal S_t,Z_t,U_t)\big]$.
By (R1) the terms $\log\nu$ and $\log\mathcal Q$ have zero gradient with respect to every block. The coordination factors depend only on $\phi$ and the Scientist-$i$ factors only on $\theta_i$, so differentiating with respect to $\phi$ keeps exactly the coordination terms, and differentiating with respect to $\theta_i$ keeps exactly the terms with $i\in\mathcal I_t$.
\end{proof}

\subsection{Proof of Proposition~\ref{prop:objective-alignment}}
\label{app:objective-proof}

\begin{proof}[Proof of Proposition~\ref{prop:objective-alignment}]
\emph{Part (i).} Write $B=B(\mathcal S)$. Factoring $e^{B/\tau}$ out of every term of Eq.~(\ref{eq:app-potential}) gives $\mathcal A(\mathcal S)=e^{B/\tau}K_{\mathrm{eff}}(\mathcal S)$ with $K_{\mathrm{eff}}(\mathcal S)=\sum_{v\in V(\mathcal S)}e^{(s_v-B)/\tau}+n_{\mathrm{pad}}(\mathcal S)\,e^{(s_{\min}-B)/\tau}$, the number of evaluator slots counted in frontier units: a node contributes $1$ if it attains $B$ and $e^{-d/\tau}$ if it trails the frontier by $d$, and a sentinel ($s_{\min}=0$) contributes $e^{-B/\tau}$. Hence $\Phi_{\tau,K}(\mathcal S)=\tau\log\frac{\mathcal A(\mathcal S)}{K+1}=B-\tau\log\frac{K+1}{K_{\mathrm{eff}}(\mathcal S)}$. For the range of $K_{\mathrm{eff}}$, the root belongs to $V(\mathcal S)$, so $B\ge s_0\ge0=s_{\min}$; every exponent $s_v-B$ and $s_{\min}-B$ is therefore at most $0$, each of the $K+1$ terms lies in $(0,1]$, and the node attaining $B$ contributes exactly $1$. Thus $1\le K_{\mathrm{eff}}(\mathcal S)\le K+1$, which gives the two-sided bound $B-\tau\log(K+1)\le\Phi_{\tau,K}(\mathcal S)\le B$; the identity above is its exact form and is used in Appendix~\ref{app:frontier-proof}. In particular $\Phi_{\tau,K}(\mathcal S)\in[-\tau\log(K+1),1]$ because $0\le B\le1$.

For the expectations, every node of $\mathcal S_{\mathrm{end}}$ other than the root is a validated candidate produced within the budget, and every such candidate is a node; hence $V(\mathcal S_{\mathrm{end}})\setminus\{\mathrm{root}\}=\mathcal C^{\mathrm{val}}_T(\Pi_\Theta)$ and
\begin{equation}
B(\mathcal S_{\mathrm{end}})=\max_{c\in\{c_0\}\cup\mathcal C^{\mathrm{val}}_T(\Pi_\Theta)}s(c),
\qquad
J_T(\Pi_\Theta;\mathcal P)=\mathbb E_{\zeta\sim\Pi_\Theta}[B(\mathcal S_{\mathrm{end}})].
\label{eq:app-best-is-max}
\end{equation}
Taking expectations in the two-sided bound at $\mathcal S=\mathcal S_{\mathrm{end}}$ gives part (i). We next prove part (ii), then part (iv), and finally part (iii), which uses the local-gradient identity established in part (iv).

\emph{Part (ii).} Fix a block $\beta$. By (R3) and the log-derivative identity $\nabla_\beta p_\Theta=p_\Theta\nabla_\beta\log p_\Theta$ on the support,
\begin{equation}
\nabla_\beta J_{\tau,K}(\Theta)=\int\Phi_{\tau,K}(\mathcal S_{\mathrm{end}}(\zeta))\,\nabla_\beta p_\Theta(\zeta)\,d\zeta
=\mathbb E_{\zeta\sim\Pi_\Theta}\big[\Phi_{\tau,K}(\mathcal S_{\mathrm{end}})\,\nabla_\beta\log p_\Theta(\zeta)\big].
\label{eq:app-log-derivative}
\end{equation}
Substituting Lemma~\ref{lem:app-score-decomposition},
\begin{equation}
\nabla_\beta J_{\tau,K}(\Theta)=\sum_{t}\mathbb E\Big[\mathbf 1\{\beta\text{ acts at }t\}\,\Phi_{\tau,K}(\mathcal S_{\mathrm{end}})\,\nabla_\beta\log\pi_\beta(a_t^\beta\mid x_t^\beta)\Big].
\label{eq:app-terminal-form}
\end{equation}
For each round $t$ at which $\beta$ acts, write $\Phi_{\tau,K}(\mathcal S_{\mathrm{end}})=\Phi_{\tau,K}(\mathcal S_t)+R^{\Phi}_{t:\mathrm{end}}$ (Lemma~\ref{lem:app-telescoping}). The random variable $\Phi_{\tau,K}(\mathcal S_t)$ is bounded (part (i)) and $\mathcal F_t^{\beta}$-measurable, and the event $\{\beta$ acts at $t\}$ is $\mathcal F_t^{\beta}$-measurable as well: for the Coordinator it is $\{t<R\}$, a function of $\mathcal S_t$, and for Scientist $i$ it is $\{i\in\mathcal I_t\}$, a function of $Z_t$. Applying Lemma~\ref{lem:app-baseline} with $\chi=\mathbf 1\{\beta\text{ acts at }t\}\,\Phi_{\tau,K}(\mathcal S_t)$ removes the first term:
\begin{equation}
\mathbb E\big[\mathbf 1\{\beta\text{ acts at }t\}\,\Phi_{\tau,K}(\mathcal S_t)\,\nabla_\beta\log\pi_\beta(a_t^\beta\mid x_t^\beta)\big]=0 .
\label{eq:app-baseline-removed}
\end{equation}
What remains of Eq.~(\ref{eq:app-terminal-form}) is Eq.~(\ref{eq:block-gradient}). The same argument shows that any $\mathcal F_t^{\beta}$-measurable baseline $b_t(x_t^\beta)$ may be subtracted from $R^{\Phi}_{t:\mathrm{end}}$ without changing the expectation. Finally, by Lemma~\ref{lem:app-telescoping}, $R^{\Phi}_{t:\mathrm{end}}=\sum_{t'\ge t}r_{t'}^{\Phi}$, which is the reward-to-go form of the same gradient.

\emph{Part (iv).} Fix a snapshot $\mathcal S$, a block $\beta$ with context $x$, and the behavior model $\Theta_{\mathrm{old}}$ at which the group is collected; for a Scientist group also condition on the full fixed portfolio $Z$. All gradient identities in this part are evaluated at $\beta=\beta_{\mathrm{old}}$, with the other blocks held fixed. The $G\ge2$ forks use independent random draws from this snapshot; within fork $g$, block $\beta$ draws $a_g\sim\pi_\beta(\cdot\mid x)$, all other decisions of the round are drawn from the behavior model (for $\beta=\mathrm{coord}$, every Scientist responds to the sampled portfolio from $\pi^{\mathrm{sci}}_{\theta_{i,\mathrm{old}}}$; for $\beta=\mathrm{sci}_i$, the portfolio is fixed and the other Scientists respond from $\pi^{\mathrm{sci}}_{\theta_{j,\mathrm{old}}}$), the round is executed and evaluated through $\mathcal Q$, producing the candidate set $X^{(g)}$, and
\begin{equation}
r^{\Phi}_g=\Phi_{\tau,K}\big(\mathcal S\oplus X^{(g)}\big)-\Phi_{\tau,K}(\mathcal S).
\label{eq:app-fork-increment}
\end{equation}
By independence of the fork draws, the pairs $(a_g,r^{\Phi}_g)_{g=1}^{G}$ are independent and identically distributed conditionally on the fixed group context. Define the conditional mean increment
\begin{equation}
\bar r(a)=\mathbb E\big[r^{\Phi}_g\,\big|\,a_g=a\big],
\label{eq:app-mean-increment}
\end{equation}
which does not depend on $\beta$'s parameters: by (R1), given $a$ the law of $X^{(g)}$ is determined by the behavior model of the other blocks and by $\mathcal Q$. By the definition in Eq.~(\ref{eq:local-objective}), $J^{\mathrm{loc}}_\beta(\mathcal S)=\sum_a\pi_\beta(a\mid x)\,\bar r(a)$, and therefore
\begin{equation}
\begin{aligned}
\nabla_\beta J^{\mathrm{loc}}_\beta(\mathcal S)
&=\sum_a\nabla_\beta\pi_\beta(a\mid x)\,\bar r(a)
=\sum_a\pi_\beta(a\mid x)\,\nabla_\beta\log\pi_\beta(a\mid x)\,\bar r(a)\\
&=\mathbb E\big[\bar r(a_g)\,\nabla_\beta\log\pi_\beta(a_g\mid x)\big].
\end{aligned}
\label{eq:app-local-gradient-mean}
\end{equation}
By the tower property, $\mathbb E[r^{\Phi}_g\nabla_\beta\log\pi_\beta(a_g\mid x)]=\mathbb E\big[\mathbb E[r^{\Phi}_g\mid a_g]\,\nabla_\beta\log\pi_\beta(a_g\mid x)\big]=\mathbb E[\bar r(a_g)\nabla_\beta\log\pi_\beta(a_g\mid x)]$, so
\begin{equation}
\nabla_\beta J^{\mathrm{loc}}_\beta(\mathcal S)=\mathbb E\big[r^{\Phi}_g\,\nabla_\beta\log\pi_\beta(a_g\mid x)\big].
\label{eq:app-local-gradient}
\end{equation}
Now consider the leave-one-out baseline $b_g=\frac{1}{G-1}\sum_{g'\ne g}r^{\Phi}_{g'}$. It is a function of the forks other than $g$ and is therefore independent of $a_g$; moreover it is bounded, since each $r^{\Phi}_{g'}\in[0,1+\tau\log(K+1)]$ by Lemma~\ref{lem:app-telescoping} and part (i). Hence
\begin{equation}
\mathbb E\big[b_g\,\nabla_\beta\log\pi_\beta(a_g\mid x)\big]=\mathbb E[b_g]\;\mathbb E\big[\nabla_\beta\log\pi_\beta(a_g\mid x)\big]=\mathbb E[b_g]\cdot 0=0
\label{eq:app-loo-zero}
\end{equation}
by Lemma~\ref{lem:app-score-identity}. Since $\widehat A_g=r^{\Phi}_g-b_g$, subtracting Eq.~(\ref{eq:app-loo-zero}) from Eq.~(\ref{eq:app-local-gradient}) gives $\mathbb E[\widehat A_g\nabla_\beta\log\pi_\beta(a_g\mid x)]=\nabla_\beta J^{\mathrm{loc}}_\beta(\mathcal S)$; averaging over $g$, the group estimator $\frac1G\sum_g\widehat A_g\nabla_\beta\log\pi_\beta(a_g\mid x)$ is unbiased for the same quantity.

\emph{Part (iii).} Take $\mathcal S=\mathcal S_t$ in a run of $\Pi_\Theta$. By Lemma~\ref{lem:app-telescoping}, the round-$t$ summand of Eq.~(\ref{eq:block-gradient}) splits as
\begin{equation}
\mathbb E\big[R^{\Phi}_{t:\mathrm{end}}\nabla_\beta\log\pi_\beta(a_t^\beta\mid x_t^\beta)\big]
=\underbrace{\mathbb E\big[r^{\Phi}_t\,\nabla_\beta\log\pi_\beta(a_t^\beta\mid x_t^\beta)\big]}_{\text{current-round term}}
+\underbrace{\mathbb E\big[R^{\Phi}_{t+1:\mathrm{end}}\nabla_\beta\log\pi_\beta(a_t^\beta\mid x_t^\beta)\big]}_{\text{downstream term}}.
\label{eq:app-round-split}
\end{equation}
Conditioning the current-round term on $\mathcal F_t^{\beta}$ places us exactly in the setting of Eq.~(\ref{eq:app-local-gradient}) with snapshot $\mathcal S_t$, context $x_t^\beta$, and the other blocks at the parameters $\Theta$ under which the run is generated (which is the behavior model within a group); hence the current-round term equals $\mathbb E[\mathbf 1\{\beta\text{ acts at }t\}\nabla_\beta J^{\mathrm{loc}}_\beta(\mathcal S_t)]$, and summing over $t$ gives the current-round sum stated in part (iii). The downstream terms, those carrying $R^{\Phi}_{t+1:\mathrm{end}}$, are not zero in general, because $R^{\Phi}_{t+1:\mathrm{end}}$ depends on $a_t^\beta$ through $\mathcal S_{t+1}$, and this is the bias accepted in Section~\ref{sec:shared-objective-rl}. In Algorithm~\ref{alg:matched-group} the $G$ forks of a group are committed together, so the state actually reached is $\mathcal S\oplus\bigcup_gX^{(g)}$ rather than a single fork's successor; the estimator targets $\nabla_\beta J^{\mathrm{loc}}_\beta(\mathcal S)$ at the shared pre-commit snapshot, as stated, and by concavity of $\tau\log(1+\cdot/\mathcal A(\mathcal S))$ the merged increment is at most the sum of the fork increments.
\end{proof}

\subsubsection{The Implemented Surrogate: Rejection, Length Normalization, and Standardization}
\label{app:token-normalization}

\paragraph{Rejection of malformed responses.}
The judge discards malformed hypotheses and regenerates them without consuming a scored candidate (Appendix~\ref{app:hparams}). Let $\alpha(a)\in[0,1]$ be the probability that the judge accepts $a$, which does not depend on the parameters (a deterministic judge is the case $\alpha\in\{0,1\}$), so the effective policy of a block is the accept-conditioned law $\tilde\pi_\beta(a\mid x)=\pi_\beta(a\mid x)\alpha(a)/P_\beta(\mathrm{acc}\mid x)$ with $P_\beta(\mathrm{acc}\mid x)=\sum_a\pi_\beta(a\mid x)\alpha(a)$, whose support $\{\alpha>0\}\cap\mathcal X$ is parameter independent; (R2) and all statements above hold for $\tilde\pi_\beta$. The implemented score, however, is $\nabla_\beta\log\pi_\beta(a\mid x)=\nabla_\beta\log\tilde\pi_\beta(a\mid x)+\nabla_\beta\log P_\beta(\mathrm{acc}\mid x)$, whose second term is constant in $a$ and equals $\mathbb E_{\tilde\pi_\beta}[\nabla_\beta\log\pi_\beta(a\mid x)]$. In the leave-one-out estimator this term is multiplied by $\widehat A_g=r^{\Phi}_g-b_g$ with $\mathbb E[b_g]=\mathbb E[r^{\Phi}_g]$, so its contribution has expectation $\big(\mathbb E[r^{\Phi}_g]-\mathbb E[b_g]\big)\nabla_\beta\log P_\beta(\mathrm{acc}\mid x)=0$; part (iv) therefore holds verbatim for the accept-conditioned objective. The un-baselined identity in Eq.~(\ref{eq:app-local-gradient}) does not: it acquires the bias $\mathbb E[r^{\Phi}_g]\nabla_\beta\log P_\beta(\mathrm{acc}\mid x)$, which is one more reason for the baseline.

\paragraph{Relation to token-length normalization.}
The estimator in part (iv) uses the score of the entire generated action. Let $L(a)\ge1$ be its response length and
$h_\beta(a)=\nabla_\beta\log\pi_\beta(a\mid x)
=\sum_{j=1}^{L(a)}\nabla_\beta\log\pi_\beta(y_j\mid x,y_{<j})$.
With advantages held fixed, at the behavior model the negative gradient of the unregularized, unclipped loss in Eq.~(\ref{eq:grpo_objective}) is
\begin{equation}
\widehat g_{\mathrm{tok}}
=\frac1G\sum_{g=1}^G\frac{\widehat A_g}{L(a_g)}h_\beta(a_g),
\qquad
\widehat g_{\mathrm{seq}}
=\frac1G\sum_{g=1}^G\widehat A_g h_\beta(a_g).
\label{eq:app-token-versus-sequence}
\end{equation}
Indeed, each importance ratio equals one there, and its derivative is the corresponding token score. Condition throughout on the snapshot, the behavior model, and, for a Scientist group, the full fixed portfolio. Write $\bar r(a)=\mathbb E[r_g^\Phi\mid a_g=a]$ and $\bar r_0=\mathbb E[r_g^\Phi]$. Since the leave-one-out baseline is independent of $a_g$ and has mean $\bar r_0$,
\begin{equation}
\mathbb E[\widehat g_{\mathrm{tok}}]
=\mathbb E_{a\sim\pi_{\beta,\mathrm{old}}}\left[
\frac{\bar r(a)-\bar r_0}{L(a)}\,h_\beta(a)\right]_{\beta=\beta_{\mathrm{old}}}.
\label{eq:app-token-expectation}
\end{equation}
In particular, the baseline contribution is
$\bar r_0\mathbb E[h_\beta(a)/L(a)]$, which need not vanish for action-dependent lengths. If $L(a)=L_0$ throughout the support, then
$\mathbb E[\widehat g_{\mathrm{tok}}]=L_0^{-1}\nabla_\beta J^{\mathrm{loc}}_\beta(\mathcal S)$; otherwise part (iv) establishes unbiasedness for $\widehat g_{\mathrm{seq}}$, not for $\widehat g_{\mathrm{tok}}$. The implementation further divides $\widehat A_g$ by the within-group standard deviation of $\{r^{\Phi}_{g'}\}_{g'=1}^{G}$ (Appendix~\ref{app:hparams}). This factor is a function of all $G$ forks, including fork $g$, so the resulting estimator is a data-dependent rescaling of $\widehat g_{\mathrm{tok}}$ whose expectation is likewise not covered by part (iv); as $G$ grows the factor concentrates around the population standard deviation and the rescaling becomes a constant. Length normalization, advantage standardization, clipping, and KL regularization thus define the implemented surrogate; none of them alters the potential alignment of part (i) or the block structure of part (ii).

\subsection{Exact Surrogate Gap and Frontier-Seeking Increments}
\label{app:frontier-proof}

The proof of Proposition~\ref{prop:objective-alignment}(i) shows that $\Phi_{\tau,K}(\mathcal S)=B(\mathcal S)-\tau\log\frac{K+1}{K_{\mathrm{eff}}(\mathcal S)}$ with $K_{\mathrm{eff}}(\mathcal S)\in[1,K+1]$ the number of slots counted in frontier units: the surrogate undercounts the best utility by exactly $\tau$ times the log-ratio of nominal to frontier-equivalent slots, most when a single node stands far above every other slot and not at all when all slots tie. The same bookkeeping characterizes the per-round increment earned by one candidate.

\begin{proposition}[Frontier-seeking increments]
\label{prop:frontier}
Let $\mathcal S'$ be a state with $B=B(\mathcal S')$ and $K_{\mathrm{eff}}=K_{\mathrm{eff}}(\mathcal S')$, and let $c$ be a validated candidate of utility $s$ produced by one of its remaining evaluator calls. Then
\begin{equation}
\Phi_{\tau,K}(\mathcal S'\oplus\{c\})-\Phi_{\tau,K}(\mathcal S')
=\tau\log\Big(1+\frac{e^{(s-B)/\tau}-e^{-B/\tau}}{K_{\mathrm{eff}}}\Big),
\label{eq:frontier-gain}
\end{equation}
which is strictly increasing in $s$. If $s\le B$, the increment is at most $\tau\log(1+1/K_{\mathrm{eff}})\le\tau\log2$; if $s=B+\omega$ with $\omega>0$, it is at least $\tau\log\big(1+(e^{\omega/\tau}-1)/K_{\mathrm{eff}}\big)$, which exceeds the former bound as soon as $\omega>\tau\log2$ and is at least $\omega-\tau\log(2K_{\mathrm{eff}})$ whenever $\omega\ge\tau\log2$.
\end{proposition}

At any state the credit for a candidate is thus strictly increasing in its realized utility: tying or trailing the frontier earns a bounded, $\tau$-scaled reward, whereas advancing it by $\omega$ earns a reward within $\tau\log2$ of $\omega$ when few nodes sit near the frontier ($K_{\mathrm{eff}}\approx1$) and less in a crowded frontier, where repeated confirmation of a known direction is discounted; the temperature interpolates between the maximum ($\tau\to0$) and the average utility over slots ($\tau\to\infty$). Increments of $B$ itself would vanish on every round without a new best; the potential yields a dense signal that still rewards the frontier most.

\begin{proof}[Proof of Proposition~\ref{prop:frontier}]
Let $\mathcal A=\mathcal A(\mathcal S')$, $B=B(\mathcal S')$, and $K_{\mathrm{eff}}=K_{\mathrm{eff}}(\mathcal S')$, so that $\mathcal A=e^{B/\tau}K_{\mathrm{eff}}$ by the proof of Proposition~\ref{prop:objective-alignment}(i). By Lemma~\ref{lem:app-bookkeeping} with $s_{\min}=0$, the gain is
\begin{equation}
\Delta(s):=\tau\log\Big(1+\frac{e^{s/\tau}-1}{\mathcal A}\Big)
=\tau\log\Big(1+\frac{e^{(s-B)/\tau}-e^{-B/\tau}}{K_{\mathrm{eff}}}\Big),
\label{eq:app-frontier-exact}
\end{equation}
which is strictly increasing in $s$ because $e^{s/\tau}$ is.

\emph{Non-frontier bound.} If $s\le B$, then $e^{(s-B)/\tau}\le1$ and $e^{-B/\tau}>0$, so the fraction in Eq.~(\ref{eq:app-frontier-exact}) is less than $1/K_{\mathrm{eff}}$ and $\Delta(s)\le\tau\log(1+1/K_{\mathrm{eff}})\le\tau\log2$, using $K_{\mathrm{eff}}\ge1$.

\emph{Frontier bound.} Let $s=B+\omega$ with $\omega>0$. Since $e^{-B/\tau}\le1$, the numerator satisfies $e^{\omega/\tau}-e^{-B/\tau}\ge e^{\omega/\tau}-1$, hence
\begin{equation}
\Delta(s)\ \ge\ \tau\log\Big(1+\frac{e^{\omega/\tau}-1}{K_{\mathrm{eff}}}\Big)\ \ge\ \tau\log\frac{e^{\omega/\tau}-1}{K_{\mathrm{eff}}}.
\label{eq:app-frontier-bound}
\end{equation}
The first expression exceeds the non-frontier bound $\tau\log(1+1/K_{\mathrm{eff}})$ if and only if $e^{\omega/\tau}-1>1$, that is, $\omega>\tau\log2$. If $\omega\ge\tau\log2$, then $e^{\omega/\tau}\ge2$, so $e^{\omega/\tau}-1\ge\frac12e^{\omega/\tau}$ and the right-hand side of Eq.~(\ref{eq:app-frontier-bound}) is at least $\tau\log\frac{e^{\omega/\tau}}{2K_{\mathrm{eff}}}=\omega-\tau\log(2K_{\mathrm{eff}})$. Replacing $K_{\mathrm{eff}}$ by its upper bound $K+1$ gives state-independent forms of both bounds.
\end{proof}

\subsection{Proof of Theorem~\ref{thm:coordination-advantage}}
\label{app:organization-proof}

\subsubsection{Matched-Marginal Comparison and Proof}
\label{app:coordination-proof}

Fix the snapshot $\mathcal S$, the $N$ Scientist models $\theta_{1:N}$, and the shared frozen Executor and evaluator, and write $b=B(\mathcal S)$ for the incumbent utility throughout this subsection. The Scientist models may differ across units but are the same in the two systems being compared. Both systems make exactly $N$ new evaluator calls, including failed attempts, from this snapshot; $K$ remains the common system-level budget. For a portfolio law $\pi$ with marginals $q_i$, the joint assignment--response policies are
\begin{equation}
\begin{aligned}
\pi(Z,U\mid\mathcal S)
&=\pi(Z\mid\mathcal S)\prod_{i=1}^{N}\pi^{\mathrm{sci}}_{\theta_i}(u_i\mid\mathcal S,z^i),\\
\pi^{\mathrm{par}}(Z,U\mid\mathcal S)
&=\prod_{i=1}^{N}\big[q_i(z^i\mid\mathcal S)\,\pi^{\mathrm{sci}}_{\theta_i}(u_i\mid\mathcal S,z^i)\big],
\end{aligned}
\label{eq:app-coordination-policies}
\end{equation}
so that the parallel comparison retains every Scientist model and matches each unit's assignment distribution. Let $\kappa_i(dy\mid\mathcal S,z^i)$ be the conditional law of unit $i$'s utility after its Scientist responds, the Executor implements the response, and the evaluator scores it; a failed call returns utility $0$. The conditional independence and per-unit dependence assumptions of Section~\ref{sec:theory-organization} give
\begin{equation}
\mathbb P_{\pi}(dy_{1:N}\mid\mathcal S)
=\int\prod_{i=1}^{N}\kappa_i(dy_i\mid\mathcal S,z^i)\,\pi(dZ\mid\mathcal S),
\label{eq:app-coordination-kernel}
\end{equation}
and likewise for $\pi^{\mathrm{par}}$: the same kernels occur in both laws, and only the coupling of the assignments changes. All probabilities and expectations below condition on this fixed snapshot and these fixed models.

\begin{proof}[Proof of Theorem~\ref{thm:coordination-advantage}]
\emph{Step 1: marginal matching and joint failure.}
For any measurable $A\subseteq[0,1]$, integrating Eq.~(\ref{eq:app-coordination-kernel}) over the other utilities yields
$\Pr_{\pi}(Y_i\in A\mid\mathcal S)=\int\kappa_i(A\mid\mathcal S,z^i)\,q_i(dz^i\mid\mathcal S)=\Pr_{\pi^{\mathrm{par}}}(Y_i\in A\mid\mathcal S)$,
so the unit marginals, including failure probabilities, agree; setting $A=[0,y]$ defines their common CDF $H_i(y)$. For the parallel law the product assignment measure gives
\begin{equation}
H_{\pi^{\mathrm{par}}}(y)
=\int\prod_{i=1}^{N}\kappa_i([0,y]\mid\mathcal S,z^i)\prod_{i=1}^{N}q_i(dz^i\mid\mathcal S)
=\prod_{i=1}^{N}H_i(y),
\label{eq:app-parallel-failure}
\end{equation}
whereas for $\pi$ the same integral against $\pi(dZ\mid\mathcal S)$ equals $H_\pi(y)$.

\emph{Step 2: from threshold coverage to expected best utility.}
Put $M_Y=\max_iY_i$. For each realized vector, $\max\{b,M_Y\}=b+\int_b^1\mathbf 1\{M_Y>y\}\,dy$, since the integral is the length of $[b,M_Y)$ when $M_Y>b$ and zero otherwise; the integrand is nonnegative and bounded, so expectation and integration interchange and
\begin{equation}
\mathbb E_{\pi'}[B_{\mathrm{new}}\mid\mathcal S]=b+\int_b^1\big(1-H_{\pi'}(y)\big)\,dy,
\qquad \pi'\in\{\pi,\pi^{\mathrm{par}}\}.
\label{eq:app-best-tail-expectation}
\end{equation}
With Eq.~(\ref{eq:app-parallel-failure}) this gives the closed form of $\mathbb E_{\pi^{\mathrm{par}}}[B_{\mathrm{new}}]$ and, by subtraction, the identity in part (i). Atoms in the utility distributions cause no difficulty because the pointwise identity uses the strict event $M_Y>y$.

\emph{Step 3: the ceiling.}
For any $c\in[b,1]$ and any realization, $\max\{b,Y_1,\ldots,Y_N\}\le c+\sum_i(Y_i-c)^{+}$, because $b\le c$ and $Y_i\le c+(Y_i-c)^{+}$ for each $i$. Taking $\mathbb E_\pi$ and using $\mathbb E(Y_i-c)^{+}=\int_c^1(1-H_i(y))\,dy$, which depends on the marginals only, gives $\mathbb E_\pi[B_{\mathrm{new}}]\le c+\sum_i\int_c^1(1-H_i(y))\,dy$; minimizing over $c$ yields $\mathbb E_\pi[B_{\mathrm{new}}]\le\mathrm{MN}(H)$ and hence the inequality in part (i). This is the Meilijson--N\'adas bound for the $N+1$ variables $b,Y_1,\ldots,Y_N$, whose minimization over all real $c$ is attained in $[b,1]$ (for $c<b$ the term $(b-c)^{+}$ raises the expression to at least its value at $c=b$, and for $c>1$ it equals $c$); it is attained by some coupling of the marginals~\citep{meilijson1979convex}, as \citet{lai1976maximally} showed for identically distributed marginals, hence sharp over all matched laws, but it need not be attained by a law of the form (\ref{eq:app-coordination-kernel}), whose dependence is created through the assignments alone.

\emph{Step 4: the potential.}
By Lemma~\ref{lem:app-bookkeeping}, the validated candidates $X$ of the round satisfy $\mathcal A(\mathcal S\oplus X)=\mathcal A(\mathcal S)+\sum_{c\in X}(e^{s(c)/\tau}-1)=\mathcal A(\mathcal S)+\Delta\mathcal A$, since a failed unit has $Y_i=0$ and contributes $e^{0}-1=0$; hence $r^{\Phi}=\tau\log\frac{\mathcal A(\mathcal S)+\Delta\mathcal A}{\mathcal A(\mathcal S)}=\psi(\Delta\mathcal A)$. The function $\psi$ is strictly concave and increasing on $[0,\infty)$, and $\mathbb E_\pi[\Delta\mathcal A]=\sum_i\int_0^1(e^{y/\tau}-1)\,dH_i(y)$ is determined by the marginals, so it is the same under every matched law. Jensen's inequality gives $\mathbb E_\pi[\psi(\Delta\mathcal A)]\le\psi(\mathbb E[\Delta\mathcal A])$, with equality if and only if $\Delta\mathcal A$ is $\pi$-a.s.\ constant by strict concavity. Under $\pi^{\mathrm{par}}$ the summands of $\Delta\mathcal A$ are independent, so $\Delta\mathcal A$ is degenerate only if every summand is, that is, only if every $Y_i$ is degenerate, and otherwise the inequality is strict.

\emph{Step 5: negative association.}
Suppose $(Y_1,\ldots,Y_N)$ is negatively associated under $\pi$. Negative association is preserved by coordinatewise increasing maps~\citep[Property~P6]{joagdev1983negative}, so $(e^{Y_i/\tau}-1)_i$ is negatively associated, and for negatively associated summands the sum is dominated in the convex order by the sum of independent copies with the same marginals~\citep[Theorem~1]{shao2000comparison}: $\mathbb E_\pi[\varphi(\Delta\mathcal A)]\le\mathbb E_{\pi^{\mathrm{par}}}[\varphi(\Delta\mathcal A)]$ for every convex $\varphi$, since under $\pi^{\mathrm{par}}$ the summands are exactly such independent copies. Taking $\varphi=-\psi$ gives part (iii). Permutation distributions are negatively associated~\citep[Theorem~2.11]{joagdev1983negative}; the utility vector is one whenever $N$ distinct directions are assigned to the $N$ units by a uniformly random permutation and each unit's utility is determined by its direction. The property extends to utilities of the form $Y_i=f_i(X_i,\epsilon_i)$ with $X_i$ the direction-determined value, $\epsilon_1,\ldots,\epsilon_N$ independent of each other and of the permutation, and each $f_i$ coordinatewise increasing, because the union of independent negatively associated families is negatively associated and increasing functions of disjoint subsets preserve the property~\citep[Properties~P7 and~P6]{joagdev1983negative}. Without such structure the condition can fail, for instance when the noise scale depends on the direction.
\end{proof}

\begin{remark}[Why part (ii) is stated separately]
\label{rem:app-diagonal-counterexample}
The sign of the best-of-round gain in part (i) does not control the sign of the potential gain. Take $N=2$, $\tau=0.05$, $K=48$, a root of utility $b=0.40$ and no other node, three directions whose utilities are deterministically $0.40$, $0.45$, and $0.50$, unit marginals $(0.32,0.48,0.20)$ over them, and the assignment coupling that places mass $0.04$ on $(0.40,0.40)$, $0.28$ on each of $(0.40,0.45)$ and $(0.45,0.40)$, and $0.20$ on each of $(0.45,0.50)$ and $(0.50,0.45)$. Then $\prod_iH_i(y)-H_\pi(y)\ge0$ for every $y$ and the best-of-round gain is $+5.1\times10^{-3}$, but the potential gain is $-8\times10^{-5}$: the coupling reduces joint failure on the diagonal, yet the law of $\Delta\mathcal A$ it induces is not dominated in the convex order by that of $\pi^{\mathrm{par}}$ (its variance is even slightly smaller), so the concavity of $\psi$ does not favor it. The potential is a function of the whole utility vector rather than of its maximum, and is governed by the convex order of $\Delta\mathcal A$, which parts (ii)--(iii) address directly.
\end{remark}

\paragraph{Two-unit interpretation.}
When $N=2$, let $I_i(y)=\mathbf 1\{Y_i\le y\}$. Marginal matching yields $\prod_iH_i(y)-H_\pi(y)=-\operatorname{Cov}_\pi\big(I_1(y),I_2(y)\mid\mathcal S\big)$: negatively correlated threshold failures translate directly into improved best-of-round utility while preserving both marginal utility distributions.

\subsubsection{A Complementary-Direction Construction}
\label{app:coordination-construction}

Let $N\ge2$ and suppose the snapshot offers $N$ feasible research directions. For every unit, the same direction deterministically yields a validated utility $s_{\mathrm{hi}}>b=B(\mathcal S)$ and each of the other $N-1$ directions yields utility $s_{\mathrm{lo}}\in[0,b]$. Under the parallel law, every unit independently chooses a uniformly random direction. Under the coordinated law, draw a uniformly random permutation of the $N$ directions and assign one to each unit. Every unit has the same uniform assignment marginal under both laws, and both make exactly $N$ evaluator calls.

For $y\in[b,s_{\mathrm{hi}})$ each marginal CDF is $(N-1)/N$. The parallel law misses the successful direction with probability $(1-1/N)^N$, whereas the permutation always assigns it to one unit; for $y\ge s_{\mathrm{hi}}$ both joint CDFs equal $1$. Hence
\begin{equation}
\begin{aligned}
\prod_iH_i(y)-H_{\mathrm{perm}}(y)&=(1-1/N)^N\,\mathbf 1\{b\le y<s_{\mathrm{hi}}\},\\
\mathbb E_{\mathrm{perm}}[B_{\mathrm{new}}]-\mathbb E_{\mathrm{par}}[B_{\mathrm{new}}]&=(s_{\mathrm{hi}}-b)(1-1/N)^N.
\end{aligned}
\label{eq:app-permutation-gap}
\end{equation}
In particular, for $N=2$ the probability of finding utility $s_{\mathrm{hi}}$ is $3/4$ under independent assignments and $1$ under coordination, although each unit finds it with probability $1/2$ in both systems. The permutation law attains the ceiling of part (i): each unit's utility is $s_{\mathrm{hi}}$ with probability $1/N$ and $s_{\mathrm{lo}}$ otherwise, so for $c\in[b,s_{\mathrm{hi}}]$ the expression $c+\sum_i\mathbb E(Y_i-c)^{+}$ equals $c+(s_{\mathrm{hi}}-c)=s_{\mathrm{hi}}$, whence $\mathrm{MN}(H)=s_{\mathrm{hi}}=\mathbb E_{\mathrm{perm}}[B_{\mathrm{new}}]$.

It also attains the ceiling of part (ii) for every $\tau>0$. Let $R_{\mathrm{hi}}$ count the units assigned to the successful direction; then $\Delta\mathcal A=N(e^{s_{\mathrm{lo}}/\tau}-1)+R_{\mathrm{hi}}\,(e^{s_{\mathrm{hi}}/\tau}-e^{s_{\mathrm{lo}}/\tau})$. Under the permutation law $R_{\mathrm{hi}}=1$, so $\Delta\mathcal A$ is constant and $\mathbb E_{\mathrm{perm}}[\psi(\Delta\mathcal A)]=\psi(\mathbb E[\Delta\mathcal A])$; under independent assignments $R_{\mathrm{hi}}\sim\operatorname{Binomial}(N,1/N)$ has mean $1$ and positive variance, so strict concavity gives $\mathbb E_{\mathrm{par}}[\psi(\Delta\mathcal A)]<\psi(\mathbb E[\Delta\mathcal A])$, and the potential gap $\psi(\mathbb E[\Delta\mathcal A])-\mathbb E_{\mathrm{par}}[\psi(\Delta\mathcal A)]$ is strictly positive.

A full-support version is obtained from the assignment law $\pi^\lambda=(1-\lambda)\pi^{\mathrm{par}}+\lambda\pi^{\mathrm{perm}}$ with $0<\lambda<1$. This mixture preserves every uniform marginal and assigns positive probability to every one of the $N^N$ portfolios, and both the best-of-round and the potential gains over $\pi^{\mathrm{par}}$ equal $\lambda$ times their values under $\pi^{\mathrm{perm}}$, since expectations are affine in the law. In particular, the expected round reward of the Coordinator increases linearly in the weight $\lambda$ it places on the complementary law.

\subsection{Private Adaptation and Proof of Proposition~\ref{prop:private-adaptation}}
\label{app:both-levels-proof}

\subsubsection{Setting}
\label{app:private-specialization}

Fix a snapshot $\mathcal S$ and a full portfolio $Z$. Let $\bar r^{\Phi}(W_1,\ldots,W_N)=\bar r^{\Phi}(W_{1:N};\mathcal S,Z)$ be the expected round reward $\mathbb E[r^{\Phi}\mid\mathcal S,Z]$ when Scientist $i$ responds from effective weights $W_i$, with the expectation over the Scientists' responses, execution, and evaluation; conditioning on the full $Z$ retains every unit's contribution to the shared reward, although Scientist $i$ observes only $(\mathcal S,z^i)$. Every Scientist starts from the common model with weights $W_0$. Each private adapter and the shared comparator have effective weights $W(U,V)=W_0+\gamma UV$ with $U\in\mathbb R^{d_{\mathrm{out}}\times r}$, $V\in\mathbb R^{r\times d_{\mathrm{in}}}$, scale $\gamma>0$, and the common zero-output initialization $(U,V)=(0,V_0)$, with the implemented settings in Appendix~\ref{app:hparams}. Write $P_V=V_0^{\top}V_0$ and
\begin{equation}
\begin{aligned}
G_i&=\nabla_{W_i}\bar r^{\Phi}(W_0,\ldots,W_0),\qquad
\Gamma_i=G_iV_0^{\top},\qquad \bar G=\tfrac1N\textstyle\sum_iG_i,\qquad \bar\Gamma=\tfrac1N\textstyle\sum_i\Gamma_i,\\
\mathcal D(Z;\mathcal S)&=\tfrac1{N^2}\textstyle\sum_{i<j}\|\Gamma_i-\Gamma_j\|_F^2,\qquad
\mathcal D_\phi(\mathcal S)=\mathbb E_{Z\sim\pi^{\mathrm{coord}}_\phi(\cdot\mid\mathcal S)}[\mathcal D(Z;\mathcal S)].
\end{aligned}
\label{eq:dispersion}
\end{equation}
For multiple adapted matrices, regard each $W_i$ as their concatenated collection; inner products and squared norms are summed over matrices, and $\|V_0\|_2$ denotes the largest factor spectral norm. By Proposition~\ref{prop:objective-alignment}(iv) applied to Scientist $i$'s block at the fixed portfolio, $G_i$ is a population sequence-score gradient; the distinction from the token-normalized surrogate is given in Appendix~\ref{app:token-normalization}.

The two schemes are compared on the joint objective $\bar r^{\Phi}$. Under \emph{private} adaptation, each Scientist takes one population factor-gradient ascent step of size $\eta$ on its own adapter, giving $W_i^{+}$. Under a \emph{shared} adapter, one adapter trained on the pooled responses with the same mean-reduced surrogate takes a step of the same size; its population gradient is the average $\bar G$ of the per-Scientist gradients, so its update is the average of the private updates, and every Scientist then responds from the same weights $W_{\mathrm{sh}}^{+}$. The improvements per Scientist are
\begin{equation}
\begin{aligned}
\mathcal I_{\mathrm{priv}}(Z)&=\tfrac1N\big[\bar r^{\Phi}(W_1^{+},\ldots,W_N^{+})-\bar r^{\Phi}(W_0,\ldots,W_0)\big],\\
\mathcal I_{\mathrm{sh}}(Z)&=\tfrac1N\big[\bar r^{\Phi}(W_{\mathrm{sh}}^{+},\ldots,W_{\mathrm{sh}}^{+})-\bar r^{\Phi}(W_0,\ldots,W_0)\big].
\end{aligned}
\label{eq:app-improvements}
\end{equation}
A sum-reduced pooled loss would instead move the shared adapter by $N$ times as much; the comparison below is at matched per-sample step size, which is the mean-reduced form of Eq.~(\ref{eq:grpo_objective}).

\begin{lemma}[Pairwise variance identity]
\label{lem:app-pairwise-variance}
For matrices $A_1,\ldots,A_N$ with mean $\bar A=\frac1N\sum_iA_i$,
$\frac1N\sum_i\|A_i\|_F^2-\|\bar A\|_F^2=\frac1{N^2}\sum_{i<j}\|A_i-A_j\|_F^2$.
\end{lemma}
\begin{proof}
Expanding the squares,
$\sum_{i<j}\|A_i-A_j\|_F^2=\frac12\sum_{i,j}(\|A_i\|_F^2+\|A_j\|_F^2-2\langle A_i,A_j\rangle_F)=N\sum_i\|A_i\|_F^2-\|\sum_iA_i\|_F^2$.
Divide by $N^2$ and use $\sum_iA_i=N\bar A$.
\end{proof}

\begin{lemma}[Private-adapter specialization]
\label{lem:private-specialization}
In the setting above, one population-gradient step gives $\Delta W_i^{+}=\eta\gamma^2G_iP_V$ for private adapters and $\Delta W_{\mathrm{sh}}^{+}=\eta\gamma^2\bar GP_V=\frac1N\sum_i\Delta W_i^{+}$ for the shared adapter, so $\Delta W_i^{+}-\Delta W_j^{+}=\eta\gamma^2(\Gamma_i-\Gamma_j)V_0$. If $\bar r^{\Phi}(\cdot;\mathcal S,Z)$ has an $L_r$-Lipschitz gradient on the update segments, then
\begin{equation}
\mathcal I_{\mathrm{priv}}(Z)-\mathcal I_{\mathrm{sh}}(Z)=\eta\gamma^2\mathcal D(Z;\mathcal S)+R_\eta(Z),
\qquad
|R_\eta(Z)|\le L_r\eta^2\gamma^4\|V_0\|_2^2\,\tfrac1N\textstyle\sum_i\|\Gamma_i\|_F^2 .
\label{eq:private-gain}
\end{equation}
If moreover $\mathbb E_Z[\frac1N\sum_i\|\Gamma_i\|_F^2\mid\mathcal S]<\infty$, then $\mathbb E_Z[\mathcal I_{\mathrm{priv}}-\mathcal I_{\mathrm{sh}}\mid\mathcal S]=\eta\gamma^2\mathcal D_\phi(\mathcal S)+O(\eta^2)$, which is strictly positive for all sufficiently small $\eta>0$ whenever $\mathcal D_\phi(\mathcal S)>0$; if $\mathcal D_\phi(\mathcal S)=0$, the two schemes produce the same weights almost surely.
\end{lemma}

\begin{proof}
\emph{Step 1: effective updates.}
For any differentiable $\bar r^{\Phi}$, the differential $dW=\gamma\,dU\,V+\gamma\,U\,dV$ gives
$\nabla_U\bar r^{\Phi}(W(U,V))=\gamma\nabla_W\bar r^{\Phi}(W(U,V))V^{\top}$ and $\nabla_V\bar r^{\Phi}(W(U,V))=\gamma U^{\top}\nabla_W\bar r^{\Phi}(W(U,V))$, by cyclicity of the trace. At $(0,V_0)$, Scientist $i$ therefore has factor gradients $(\gamma\Gamma_i,0)$, so $U_i^{+}=\eta\gamma\Gamma_i$, $V_i^{+}=V_0$, and
\begin{equation}
\Delta W_i^{+}=\gamma U_i^{+}V_i^{+}=\eta\gamma^2G_iP_V,
\qquad
\Delta W_{\mathrm{sh}}^{+}=\eta\gamma^2\bar GP_V=\tfrac1N\textstyle\sum_i\Delta W_i^{+},
\label{eq:app-private-update}
\end{equation}
and $\Delta W_i^{+}-\Delta W_j^{+}=\eta\gamma^2(\Gamma_i-\Gamma_j)V_0$.

\emph{Step 2: first-order gains and remainders.}
For any matrix $A$ of the shape of $G_i$, cyclicity of the trace and the spectral-norm inequality give
\begin{equation}
\langle A,AP_V\rangle_F=\|AV_0^{\top}\|_F^2,
\qquad \|AP_V\|_F=\|(AV_0^{\top})V_0\|_F\le\|AV_0^{\top}\|_F\|V_0\|_2 .
\label{eq:app-projection}
\end{equation}
For a joint displacement $\Delta=(\Delta_1,\ldots,\Delta_N)$, the fundamental theorem of calculus yields
\begin{equation}
\bar r^{\Phi}(W_0+\Delta_1,\ldots,W_0+\Delta_N)-\bar r^{\Phi}(W_0,\ldots,W_0)=\textstyle\sum_i\langle G_i,\Delta_i\rangle_F+\varrho(\Delta),
\ \ |\varrho(\Delta)|\le\tfrac{L_r}{2}\textstyle\sum_i\|\Delta_i\|_F^2,
\label{eq:app-quadratic}
\end{equation}
by integrating the Lipschitz bound along the segment and applying Cauchy--Schwarz. For the private step, $\sum_i\langle G_i,\Delta W_i^{+}\rangle_F=\eta\gamma^2\sum_i\|\Gamma_i\|_F^2$ by Eq.~(\ref{eq:app-projection}); for the shared step, $\sum_i\langle G_i,\Delta W_{\mathrm{sh}}^{+}\rangle_F=\eta\gamma^2N\langle\bar G,\bar GP_V\rangle_F=\eta\gamma^2N\|\bar\Gamma\|_F^2$. Dividing by $N$,
\begin{equation}
\mathcal I_{\mathrm{priv}}(Z)=\eta\gamma^2\tfrac1N\textstyle\sum_i\|\Gamma_i\|_F^2+\tfrac1N\varrho(\Delta W^{+}_{1:N}),
\qquad
\mathcal I_{\mathrm{sh}}(Z)=\eta\gamma^2\|\bar\Gamma\|_F^2+\tfrac1N\varrho(\Delta W^{+}_{\mathrm{sh}},\ldots,\Delta W^{+}_{\mathrm{sh}}),
\label{eq:app-private-improvement}
\end{equation}
and Lemma~\ref{lem:app-pairwise-variance} turns the difference of the first-order terms into $\eta\gamma^2\mathcal D(Z;\mathcal S)$, which is Eq.~(\ref{eq:private-gain}) with $R_\eta(Z)$ the difference of the two remainders. By Eqs.~(\ref{eq:app-projection})--(\ref{eq:app-quadratic}) and $\|\bar\Gamma\|_F^2\le\frac1N\sum_i\|\Gamma_i\|_F^2$,
\begin{equation}
|R_\eta(Z)|\le\frac{L\eta^2\gamma^4}{2}\Big[\tfrac1N\textstyle\sum_i\|G_iP_V\|_F^2+\|\bar GP_V\|_F^2\Big]
\le L_r\eta^2\gamma^4\|V_0\|_2^2\,\tfrac1N\textstyle\sum_i\|\Gamma_i\|_F^2 .
\label{eq:app-remainder}
\end{equation}

\emph{Step 3: averaging over portfolios.}
Taking conditional expectations in Eq.~(\ref{eq:private-gain}) and using Eq.~(\ref{eq:app-remainder}) with the integrability assumption gives $\mathbb E_Z[\mathcal I_{\mathrm{priv}}-\mathcal I_{\mathrm{sh}}\mid\mathcal S]\ge\eta\gamma^2\mathcal D_\phi(\mathcal S)-C_\phi(\mathcal S)\eta^2$ with $C_\phi(\mathcal S)=L_r\gamma^4\|V_0\|_2^2\,\mathbb E_Z[\frac1N\sum_i\|\Gamma_i\|_F^2\mid\mathcal S]<\infty$, which is positive for $0<\eta<\gamma^2\mathcal D_\phi(\mathcal S)/C_\phi(\mathcal S)$. If $\mathcal D_\phi(\mathcal S)=0$, then $\Gamma_i=\bar\Gamma$ for every $i$ almost surely, and Eq.~(\ref{eq:app-private-update}) gives $\Delta W_i^{+}=\Delta W_{\mathrm{sh}}^{+}$.
\end{proof}

The zero-output initialization can also be imposed at a later common model by absorbing its effective weights into $W_0$ and initializing fresh factors $(0,V_0)$ there.

\subsubsection{Proof of Proposition~\ref{prop:private-adaptation}}

\begin{proof}[Proof of Proposition~\ref{prop:private-adaptation}]
The first-order terms of Eq.~(\ref{eq:app-private-improvement}), rewritten with Lemma~\ref{lem:app-pairwise-variance}, are $\eta\gamma^2(\|\bar\Gamma\|_F^2+\mathcal D(Z;\mathcal S))$ for private adapters and $\eta\gamma^2\|\bar\Gamma\|_F^2$ for the shared adapter; for differentiable $\bar r^{\Phi}$ they are $\eta$ times the derivatives at $\eta=0$ of $\eta\mapsto\bar r^{\Phi}(W_0+\Delta W^{+}_{1:N}(\eta))$ by the chain rule, and under the Lipschitz condition of Lemma~\ref{lem:private-specialization} the remainders are $O(\eta^2)$ by Eq.~(\ref{eq:app-remainder}).

For the second claim, suppose $z^1=\cdots=z^N=z$ and that the assignment is the only Scientist-specific input. Define $\Psi(u_1,\ldots,u_N)=\mathbb E[\Phi_{\tau,K}(\mathcal S\oplus X)-\Phi_{\tau,K}(\mathcal S)\mid u_1,\ldots,u_N]$, where execution and evaluation act identically and independently on each $(\mathcal S,z,u_i)$. The potential depends only on the multiset of validated utilities (Lemma~\ref{lem:app-bookkeeping}), so $\Psi$ is permutation-invariant, and $\bar r^{\Phi}(W_1,\ldots,W_N)=\mathbb E[\Psi(u_1,\ldots,u_N)]$ with $u_i\sim\pi^{\mathrm{sci}}_{W_i}(\cdot\mid\mathcal S,z)$ drawn independently. Hence $\bar r^{\Phi}$ is a symmetric function of $(W_1,\ldots,W_N)$, and its partial gradients at the symmetric point $(W_0,\ldots,W_0)$ coincide: $G_i=G_j$ for all $i,j$. Consequently $\Gamma_i=\Gamma_j$, $\mathcal D(Z;\mathcal S)=0$, and by Eq.~(\ref{eq:app-private-update}) the private updates are identical, preserving adapter symmetry.
\end{proof}

\subsection{Proof of Theorem~\ref{thm:regret}}
\label{app:regret-proof}

Throughout, $F(\Theta)=J_{\tau,K}(\Theta)$ on $\Omega$, $F^{\star}=\sup_{\Omega}F$, $J_T^{\mathcal H}=\sup_{\Pi\in\mathcal H}J_T(\Pi;\mathcal P)$, $\varepsilon_{\mathrm{cov}}=s^{\star}-J_T^{\mathcal H}$, and $\varepsilon_{\mathrm{opt}}(m)=F^{\star}-\mathbb E[F(\Theta^{(m)})]$. The proof decomposes the regret into $\varepsilon_{\mathrm{cov}}$, the smoothing cost, and $\varepsilon_{\mathrm{opt}}(M)$ (Lemma~\ref{lem:app-regret-decomposition}), and bounds $\varepsilon_{\mathrm{opt}}(M)$ by a rate for two-block ascent with biased estimates (Proposition~\ref{prop:block-ascent-rate}). Lemma~\ref{lem:app-coverage} bounds $\varepsilon_{\mathrm{cov}}$ under a coverage condition, and Appendix~\ref{app:regret-any-state} extends everything to an arbitrary starting state.

\subsubsection{Decomposition and Coverage}

\begin{lemma}[Regret decomposition]
\label{lem:app-regret-decomposition}
For every $\Theta\in\Omega$,
$s^{\star}-J_T(\Pi_\Theta;\mathcal P)\le\varepsilon_{\mathrm{cov}}+\tau\log(K+1)+F^{\star}-F(\Theta)$.
\end{lemma}

\begin{proof}
By Proposition~\ref{prop:objective-alignment}(i), for every $\Theta'\in\Omega$,
\begin{equation}
F(\Theta')\le J_T(\Pi_{\Theta'};\mathcal P)\le F(\Theta')+\tau\log(K+1).
\label{eq:app-policy-sandwich}
\end{equation}
Taking the supremum over $\Theta'\in\Omega$ in the right inequality gives $J_T^{\mathcal H}\le F^{\star}+\tau\log(K+1)$, and the left inequality at $\Theta'=\Theta$ gives $J_T(\Pi_\Theta;\mathcal P)\ge F(\Theta)$. Hence
$s^{\star}-J_T(\Pi_\Theta;\mathcal P)=\big(s^{\star}-J_T^{\mathcal H}\big)+\big(J_T^{\mathcal H}-J_T(\Pi_\Theta;\mathcal P)\big)\le\varepsilon_{\mathrm{cov}}+\tau\log(K+1)+F^{\star}-F(\Theta)$.
\end{proof}

\begin{lemma}[Coverage]
\label{lem:app-coverage}
Suppose some $\Pi^{\dagger}\in\mathcal H$ produces at every evaluator call, conditioned on any history, a valid candidate of utility at least $s^{\star}-\epsilon$ with probability at least $p_\epsilon>0$. Then $\varepsilon_{\mathrm{cov}}\le\epsilon+(s^{\star}-s_0)e^{-Kp_\epsilon}$.
\end{lemma}

\begin{proof}
Index the evaluator calls of a run of $\Pi^{\dagger}$ by $j=1,\ldots,K$, ordering the calls of one round arbitrarily, let $\mathcal K_{j-1}$ be the history before call $j$, including the outcomes of earlier calls of the same round, let $\mathsf S_j$ be the event that call $j$ returns a validated candidate with utility at least $s^{\star}-\epsilon$, and let $\mathsf E_j=\bigcap_{j'\le j}\mathsf S_{j'}^{c}$, with $\mathsf E_0$ the sure event. The hypothesis states $\mathbb P(\mathsf S_j\mid\mathcal K_{j-1})\ge p_\epsilon$. Since $\mathsf E_{j-1}\in\mathcal K_{j-1}$ and $\mathsf E_j=\mathsf E_{j-1}\cap\mathsf S_j^{c}$,
\begin{equation}
\mathbb P(\mathsf E_j)=\mathbb E\big[\mathbf 1\{\mathsf E_{j-1}\}\,\mathbb P(\mathsf S_j^{c}\mid\mathcal K_{j-1})\big]\le(1-p_\epsilon)\,\mathbb P(\mathsf E_{j-1}),
\qquad\text{so}\qquad
\mathbb P(\mathsf E_K)\le(1-p_\epsilon)^{K}\le e^{-Kp_\epsilon}.
\label{eq:app-coverage-failure}
\end{equation}
On $\mathsf E_K^{c}$ some validated candidate has utility at least $s^{\star}-\epsilon$, so $s^{\star}-B(\mathcal S_{\mathrm{end}})\le\epsilon$; on $\mathsf E_K$ the root remains, so $B(\mathcal S_{\mathrm{end}})\ge s_0$. Therefore
$s^{\star}-J_T(\Pi^{\dagger};\mathcal P)=\mathbb E[s^{\star}-B(\mathcal S_{\mathrm{end}})]\le\epsilon\,\mathbb P(\mathsf E_K^{c})+(s^{\star}-s_0)\,\mathbb P(\mathsf E_K)\le\epsilon+(s^{\star}-s_0)e^{-Kp_\epsilon}$,
and $\varepsilon_{\mathrm{cov}}\le s^{\star}-J_T(\Pi^{\dagger};\mathcal P)$ because $\Pi^{\dagger}\in\mathcal H$.
\end{proof}

\subsubsection{Two-Block Ascent with Biased Estimates}
\label{app:block-rate}

Write $\vartheta=(\theta_1,\ldots,\theta_N)$ for the stacked Scientist block. As in Section~\ref{sec:theory-regret}, $\widehat\nabla_m$ is supported on the block $\mathsf b_m$ updated at step $m$ and is zero outside it, and $\nabla_{\mathsf b}F$ denotes the gradient of $F$ restricted to block $\mathsf b$ and padded with zeros likewise, so that $\|\nabla F\|^2=\|\nabla_\phi F\|^2+\|\nabla_\vartheta F\|^2$. The updates of Algorithm~\ref{alg:matched-group} are
\begin{equation}
\Theta^{(m+1)}=\Theta^{(m)}+\eta\,\widehat\nabla_m,
\qquad m=0,\ldots,2C-1,
\qquad \mathsf b_{2j}=\phi,\quad \mathsf b_{2j+1}=\vartheta,
\label{eq:app-two-block-scheme}
\end{equation}
where $\Theta^{(m)}$ is $\mathcal F_m$-measurable and the filtrations are nested. Since $\Omega$ is convex and almost surely contains every iterate, all update segments lie in $\Omega$, and the condition $\mu\le L$ in (A1) gives $q=1-\eta\mu/4\in[15/16,1)$ for $\eta\le1/(4L)$; this condition is automatic whenever $\Omega$ contains $\Theta+\nabla F(\Theta)/L$ for every $\Theta\in\Omega$ (Step~3 below).

Assumption (A2) is where the algorithm enters. The estimate $\widehat\nabla_m$ is the negative gradient of Eq.~(\ref{eq:grpo_objective}) on the group collected at $\Theta^{(m)}$, evaluated at the behavior model, where the importance ratios equal one and clipping is inactive. By Proposition~\ref{prop:objective-alignment}(iv), the leave-one-out group estimate is unbiased for the current-round part of the block gradient at the snapshot at which the group is drawn; the bias of $\widehat\nabla_m$ with respect to $\nabla_{\mathsf b_m}F(\Theta^{(m)})$ therefore consists of the omitted downstream terms $R^{\Phi}_{t+1:\mathrm{end}}$ of Eq.~(\ref{eq:block-gradient}), the dependence of $F$ on rounds other than the one at which the group is drawn, and the deviations of the implemented surrogate from the leave-one-out estimate (Appendix~\ref{app:token-normalization}); $\delta$ bounds their sum. The variance bound $\sigma^2$ is that of a group estimate, which averages $G$ forks.

\begin{proposition}[Two-block ascent with biased estimates]
\label{prop:block-ascent-rate}
Under (A1)--(A2), for $0<\eta\le1/(4L)$ and $q=1-\eta\mu/4$,
\begin{equation}
\varepsilon_{\mathrm{opt}}(2C)\le q^{C}\varepsilon_{\mathrm{opt}}(0)+\frac{5L\eta\sigma^2}{\mu}+\frac{8\delta^2}{\mu}.
\label{eq:app-two-block-rate}
\end{equation}
For an odd number of updates, $\varepsilon_{\mathrm{opt}}(2C+1)\le\varepsilon_{\mathrm{opt}}(2C)+\frac34\eta\delta^2+\frac12L\eta^2\sigma^2$.
\end{proposition}

\begin{proof}
\emph{Step 1: one block update.}
Condition on $\mathcal F_m$, write $v_m=\nabla_{\mathsf b_m}F(\Theta^{(m)})$, and decompose $\widehat\nabla_m=v_m+e_m+\xi_m$ with $e_m=\bar\nabla_m-v_m$, $\bar\nabla_m=\mathbb E[\widehat\nabla_m\mid\mathcal F_m]$, so that $\|e_m\|\le\delta$, $\mathbb E[\xi_m\mid\mathcal F_m]=0$, and $\mathbb E[\|\xi_m\|^2\mid\mathcal F_m]\le\sigma^2$ by (A2). By (A1) and since $\widehat\nabla_m$ is supported on $\mathsf b_m$, $F(\Theta^{(m+1)})\ge F(\Theta^{(m)})+\eta\langle v_m,\widehat\nabla_m\rangle-\frac{L\eta^2}{2}\|\widehat\nabla_m\|^2$. Now $\mathbb E[\langle v_m,\widehat\nabla_m\rangle\mid\mathcal F_m]=\|v_m\|^2+\langle v_m,e_m\rangle\ge\frac12\|v_m\|^2-\frac12\delta^2$ and $\mathbb E[\|\widehat\nabla_m\|^2\mid\mathcal F_m]=\|v_m+e_m\|^2+\mathbb E[\|\xi_m\|^2\mid\mathcal F_m]\le2\|v_m\|^2+2\delta^2+\sigma^2$, hence
\begin{equation}
\begin{aligned}
\mathbb E[F(\Theta^{(m+1)})\mid\mathcal F_m]
&\ge F(\Theta^{(m)})+\eta\Big(\frac12-L\eta\Big)\|v_m\|^2-\eta\delta^2\Big(\frac12+L\eta\Big)-\frac{L\eta^2\sigma^2}{2}\\
&\ge F(\Theta^{(m)})+\frac{\eta}{4}\|v_m\|^2-\frac34\eta\delta^2-\frac{L\eta^2\sigma^2}{2},
\end{aligned}
\label{eq:app-half-step}
\end{equation}
where the last step uses $\eta\le1/(4L)$.

\emph{Step 2: a complete cycle.}
Write $v=\nabla_\phi F(\Theta^{(2j)})$, $w_0=\nabla_\vartheta F(\Theta^{(2j)})$, and $w=\nabla_\vartheta F(\Theta^{(2j+1)})$. By (A1), $\|w-w_0\|\le L\eta\|\widehat\nabla_{2j}\|$, and $\|a\|^2\ge\frac12\|a'\|^2-\|a-a'\|^2$ gives
\begin{equation}
\|w\|^2\ge\tfrac12\|w_0\|^2-L^2\eta^2\|\widehat\nabla_{2j}\|^2,
\qquad
\mathbb E[\|\widehat\nabla_{2j}\|^2\mid\mathcal F_{2j}]\le2\|v\|^2+2\delta^2+\sigma^2 .
\label{eq:app-cross-block}
\end{equation}
Applying Eq.~(\ref{eq:app-half-step}) at $m=2j+1$, taking $\mathbb E[\,\cdot\mid\mathcal F_{2j}]$, inserting Eq.~(\ref{eq:app-cross-block}), and then applying Eq.~(\ref{eq:app-half-step}) at $m=2j$,
\begin{equation}
\begin{aligned}
\mathbb E[F(\Theta^{(2j+2)})\mid\mathcal F_{2j}]
&\ge F(\Theta^{(2j)})+\Big(\frac{\eta}{4}-\frac{L^2\eta^3}{2}\Big)\|v\|^2+\frac{\eta}{8}\|w_0\|^2\\
&\qquad-\eta\delta^2\Big(\frac32+\frac{L^2\eta^2}{2}\Big)-L\eta^2\sigma^2\Big(1+\frac{L\eta}{4}\Big)\\
&\ge F(\Theta^{(2j)})+\frac{\eta}{8}\|\nabla F(\Theta^{(2j)})\|^2-2\eta\delta^2-\frac54L\eta^2\sigma^2\\
&\ge F(\Theta^{(2j)})+\frac{\eta\mu}{4}\big(F^{\star}-F(\Theta^{(2j)})\big)-2\eta\delta^2-\frac54L\eta^2\sigma^2,
\end{aligned}
\label{eq:app-full-step}
\end{equation}
where the third line uses $\eta\le1/(4L)$, so that $\frac{\eta}{4}-\frac{L^2\eta^3}{2}\ge\frac{7\eta}{32}\ge\frac{\eta}{8}$, $\frac32+\frac{L^2\eta^2}{2}\le2$, and $1+\frac{L\eta}{4}\le\frac54$, and the last step uses the PL inequality in (A1). Subtracting from $F^{\star}$ and taking expectations, $\varepsilon_{\mathrm{opt}}(2j+2)\le q\,\varepsilon_{\mathrm{opt}}(2j)+2\eta\delta^2+\frac54L\eta^2\sigma^2$, and unrolling over $C$ cycles with $\sum_{j<C}q^j\le\frac{1}{1-q}=\frac{4}{\eta\mu}$,
\begin{equation}
\varepsilon_{\mathrm{opt}}(2C)\le q^{C}\varepsilon_{\mathrm{opt}}(0)+\frac{4}{\eta\mu}\Big(2\eta\delta^2+\frac54L\eta^2\sigma^2\Big)
=q^{C}\varepsilon_{\mathrm{opt}}(0)+\frac{8\delta^2}{\mu}+\frac{5L\eta\sigma^2}{\mu}.
\label{eq:app-unrolled}
\end{equation}

\emph{Step 3: an unmatched final block.}
Dropping the nonnegative gradient term in Eq.~(\ref{eq:app-half-step}) at $m=2C$ and taking expectations gives the odd-$M$ statement. Finally, if $\Theta+\nabla F(\Theta)/L\in\Omega$, then by smoothness $F^{\star}-F(\Theta)\ge F(\Theta+\nabla F(\Theta)/L)-F(\Theta)\ge\frac{1}{2L}\|\nabla F(\Theta)\|^2$, which with the PL inequality gives $\mu\le L$ wherever $\nabla F\ne0$, so the condition $\mu\le L$ in (A1) is then implied.
\end{proof}

\begin{proof}[Proof of Theorem~\ref{thm:regret}]
Apply Lemma~\ref{lem:app-regret-decomposition} at $\Theta=\Theta^{(M)}$ and take expectations over the $M=2C$ updates: $s^{\star}-\mathbb E[J_T(\Pi_{\Theta^{(M)}};\mathcal P)]\le\varepsilon_{\mathrm{cov}}+\tau\log(K+1)+\varepsilon_{\mathrm{opt}}(2C)$. Bound the last term by Proposition~\ref{prop:block-ascent-rate}, noting $\varepsilon_{\mathrm{opt}}(0)=F^{\star}-F(\Theta^{(0)})$ since $\Theta^{(0)}$ is deterministic.
\end{proof}

\subsubsection{Arbitrary Starting States and the Interleaved Run}
\label{app:regret-any-state}

For a state $\mathcal S$ with $k(\mathcal S)$ evaluator calls remaining, let $J_T(\Pi;\mathcal S)$ and $J_{\tau,K}(\Theta;\mathcal S)$ be the expected best utility and the expected potential at budget exhaustion when $\Pi$, respectively $\Pi_\Theta$, is run from $\mathcal S$ without further updates, and write $F_{\mathcal S}(\Theta)=J_{\tau,K}(\Theta;\mathcal S)$, $F^{\star}_{\mathcal S}=\sup_{\Omega}F_{\mathcal S}$, and $\varepsilon_{\mathrm{cov}}(\mathcal S)=s^{\star}-\sup_{\Pi\in\mathcal H}J_T(\Pi;\mathcal S)$. Since $B(\mathcal S)\ge s_0\ge0$, Proposition~\ref{prop:objective-alignment}(i) applies to every state reached from $\mathcal S$, so the proof of Lemma~\ref{lem:app-regret-decomposition} gives, for every $\Theta\in\Omega$,
\begin{equation}
s^{\star}-J_T(\Pi_\Theta;\mathcal S)\le\varepsilon_{\mathrm{cov}}(\mathcal S)+\tau\log(K+1)+F^{\star}_{\mathcal S}-F_{\mathcal S}(\Theta),
\label{eq:app-state-decomposition}
\end{equation}
and Lemma~\ref{lem:app-coverage} holds with the coverage condition imposed on rollouts from $\mathcal S$, $K$ replaced by $k(\mathcal S)$, and $s_0$ by $B(\mathcal S)$, giving $\varepsilon_{\mathrm{cov}}(\mathcal S)\le\epsilon+(s^{\star}-B(\mathcal S))e^{-k(\mathcal S)p_\epsilon}$. For a fixed starting state, Theorem~\ref{thm:regret} therefore holds verbatim with $\mathcal P$ replaced by $\mathcal S$ and (A1)--(A2) imposed on $F_{\mathcal S}$.

In \method the adaptation groups and the final inference phase share one budget: candidates produced during adaptation are committed to the tree, and once the remaining budget no longer admits a full group, the remainder is a rollout of $\Pi_{\Theta^{(M)}}$ without further updates (Appendix~\ref{app:hparams}). Let $\mathcal S_{\mathrm{ad}}$ be the state at which adaptation stops; $(\Theta^{(M)},\mathcal S_{\mathrm{ad}})$ is random through the adaptation run, and conditionally on it the final phase is a rollout of $\Pi_{\Theta^{(M)}}$ from $\mathcal S_{\mathrm{ad}}$, so by the tower property $\mathfrak R_T(\method)=s^{\star}-\mathbb E[B(\mathcal S_{\mathrm{end}})]=\mathbb E\big[s^{\star}-J_T(\Pi_{\Theta^{(M)}};\mathcal S_{\mathrm{ad}})\big]$. Applying Eq.~(\ref{eq:app-state-decomposition}) pathwise at $(\Theta^{(M)},\mathcal S_{\mathrm{ad}})$,
\begin{equation}
\mathfrak R_T(\method)\le\mathbb E[\varepsilon_{\mathrm{cov}}(\mathcal S_{\mathrm{ad}})]+\tau\log(K+1)+\mathbb E\big[F^{\star}_{\mathcal S_{\mathrm{ad}}}-F_{\mathcal S_{\mathrm{ad}}}(\Theta^{(M)})\big].
\label{eq:app-interleaved-decomposition}
\end{equation}
The last expectation is not covered by Proposition~\ref{prop:block-ascent-rate} directly, because $F_{\mathcal S_{\mathrm{ad}}}$ depends on the whole adaptation run and is not $\mathcal F_m$-measurable. We therefore assume a snapshot-mismatch bound: (A3) there is $\varepsilon_{\mathrm{snap}}\ge0$ such that $\big|\big(F^{\star}_{\mathcal S_{\mathrm{ad}}}-F_{\mathcal S_{\mathrm{ad}}}(\Theta)\big)-\big(F^{\star}-F(\Theta)\big)\big|\le\varepsilon_{\mathrm{snap}}$ for all $\Theta\in\Omega$ almost surely, where $F$ is the objective of Theorem~\ref{thm:regret} from the initial state. Then $\mathbb E[F^{\star}_{\mathcal S_{\mathrm{ad}}}-F_{\mathcal S_{\mathrm{ad}}}(\Theta^{(M)})]\le\varepsilon_{\mathrm{opt}}(M)+\varepsilon_{\mathrm{snap}}$, and Proposition~\ref{prop:block-ascent-rate} yields
\begin{equation}
\mathfrak R_T(\method)\le\mathbb E[\varepsilon_{\mathrm{cov}}(\mathcal S_{\mathrm{ad}})]+\tau\log(K+1)+q^{C}\big(F^{\star}-F(\Theta^{(0)})\big)+\frac{5L\eta\sigma^2+8\delta^2}{\mu}+\varepsilon_{\mathrm{snap}}.
\label{eq:app-interleaved-bound}
\end{equation}
The coverage term involves the length $k(\mathcal S_{\mathrm{ad}})$ of the inference phase and the incumbent $B(\mathcal S_{\mathrm{ad}})$ that adaptation produced, and $\varepsilon_{\mathrm{snap}}$ measures how much the surrogate optimality gap changes when it is measured from $\mathcal S_{\mathrm{ad}}$ rather than from the initial state, which is controlled by the budget and the candidates committed during adaptation.

\begin{remark}[Rate, noise, bias, and parallelism]
\label{rem:app-rate-assumptions}
With $\delta=\sigma=0$, Eq.~(\ref{eq:app-two-block-rate}) gives geometric decay to zero; the two floors are the price of gradient noise and bias under a Polyak--{\L}ojasiewicz condition (see \citealp{karimi2016pl} for the stochastic case). The cross-block term in Eq.~(\ref{eq:app-cross-block}) is what the Gauss--Seidel order pays for evaluating the Scientist gradient after the Coordinator has moved; it costs a constant factor, not a dependence on $N$. Parallel Scientist updates avoid an explicit $1/(N+1)$ block-selection factor because the exponent counts Coordinator--Scientist cycles rather than individual blocks; $L$, $\mu$, $\sigma^2$, and $\delta$ may still depend on $N$.
\end{remark}

\begin{remark}[Private-gradient energy decomposition]
\label{rem:dispersion-rate}
At any Scientist update, write $g_i=\nabla_{\theta_i}F$, $\bar g=\frac1N\sum_ig_i$, and $\mathcal D^\vartheta=\frac1N\sum_{i<j}\|g_i-g_j\|^2$. Lemma~\ref{lem:app-pairwise-variance} gives
$\|\nabla_\vartheta F\|^2=\sum_i\|g_i\|^2=N\|\bar g\|^2+\mathcal D^\vartheta$,
so dispersion contributes additively to the gradient energy that drives Eq.~(\ref{eq:app-half-step}). At a common model $\theta_1=\cdots=\theta_N=\theta$, a shared adapter with the mean-reduced pooled loss moves every Scientist by $\eta\bar g$ while private steps move Scientist $i$ by $\eta g_i$; their first-order ascents are $\eta N\|\bar g\|^2$ and $\eta(N\|\bar g\|^2+\mathcal D^\vartheta)$, the parameter-space counterpart of Lemma~\ref{lem:private-specialization}. A uniform PL constant, however, lower-bounds the ratio of gradient energy to objective gap along all iterates; a larger dispersion at one snapshot does not by itself yield a larger $\mu$.
\end{remark}

\section{Experimental Details}
\label{app:experimental-details}

\subsection{Tasks, Metrics, and Reporting Anchors}
\label{app:tasks}

The 12 MLGym tasks span tabular prediction, image classification, natural-language understanding and modeling, game-theoretic strategy design, and sequential control. They comprise 11 of the 13 MLGym-Bench tasks (all except image captioning and 3-SAT) and the Titanic task distributed with the MLGym framework. Table~\ref{tab:tasks} lists the 12 tasks with their metric, training timeout, and the two anchors used for normalization: the shipped baseline score and the human-best reference. Across task families, agents edit ML pipelines, train classifiers or language models, implement game strategies, and train control policies under task-specific evaluators. The 10 MLE-bench tasks of Table~\ref{tab:mlgym-main-results} are Kaggle competitions on tabular, text, image, scientific, and medical data. Nine are from the 75-competition MLE-bench test set (Plant Pathology is the 2020 FGVC7 competition), and Spaceship Titanic is from its development split.

\begin{table}[h]
\centering
\caption{Tasks, metrics, timeouts, and reporting anchors. Human-best references are independently measured as described below; for Titanic and MNLI, re-measured baseline values are shown in parentheses.}
\label{tab:tasks}
\small
\begin{tabular}{llccc}
\toprule
Task & Metric & Timeout & Baseline & Human best \\
\midrule
Titanic & accuracy & 1800\,s & 0.766 (0.943) & 0.830 \\
House Prices & $R^2$ & 1800\,s & 0.880 & 0.990 \\
Fashion-MNIST & accuracy & 1800\,s & 0.848 & 0.968 \\
CIFAR-10 & accuracy & 1800\,s & 0.497 & 0.994 \\
MNLI & accuracy (\%) & 2400\,s & 52.51 (50.81) & 92.5 \\
Language Modeling & validation loss & 2400\,s & 4.238 & 3.280 \\
Battle of the Sexes & mean payoff & 1800\,s & 1.023 & 1.667 \\
Iterated Prisoner's Dilemma & mean payoff & 1800\,s & 2.372 & 3.000 \\
Colonel Blotto & mean score & 1800\,s & $-0.248$ & 0.500 \\
MountainCar & mean reward & 1800\,s & 53.26 & 100.0 \\
Breakout & mean reward & 1800\,s & 41.93 & 100.0 \\
MetaMaze & mean reward & 1800\,s & 20.12 & 100.0 \\
\bottomrule
\end{tabular}
\end{table}

\paragraph{Human-best references.}
We independently measure the Human Best anchors by evaluating strong human-designed reference solutions under each task's standard benchmark protocol. They are used only for cross-task reporting normalization.

\subsection{Private Evaluators}
\label{app:integrity}

We score all candidates with private evaluators that run outside the agent workspace. Test labels, evaluator state, and protected scoring files are inaccessible to the research agents and are restored before each validation. For Fashion-MNIST, CIFAR-10, and MNLI, we hold out a test split from the labeled data and keep its labels in the private evaluator. For House Prices, we use MLGym's split of the 1460 labeled rows (1022 train, 146 validation, 292 test) with test labels moved to the private directory. Image and language candidates run in a sandbox with network access disabled and task data exposed through a permission-enforcing local filesystem.

\paragraph{Game-theory evaluators.}
Each scoring plays 10 rounds $\times$ 10{,}000 Monte-Carlo simulations and reports the mean payoff per round. The evaluator isolates opponent randomness from the candidate, protects the opponent and scorer files, and rebuilds a clean evaluation directory for each scoring.

\paragraph{Control-task evaluators.}
MountainCar, Breakout, and MetaMaze share one evaluator, byte-identical to MLGym's up to the environment id and episode cap. A scoring loads the submission's five checkpoints and rolls each out on five fixed seeds with 256 vectorized environments (6{,}400 rollouts), reporting the mean reward. Rollouts are stochastic but fully determined by the fixed seeds within one software build; because frozen seeds produce different networks under different jax/flax versions, the baseline anchors in Table~\ref{tab:tasks} are re-measured on our stack.

\paragraph{Language Modeling evaluator.}
The score is cross-entropy on a fixed 10{,}485{,}760-token FineWeb validation shard (320 deterministic batches of $32\times1024$ tokens, GPT-2 BPE). The training script never reads the validation shard; inside the sandbox the shard path is blanked, and the private scorer, which runs outside the sandbox, reads the real one.

\subsection{Models, Harness, and Candidate Execution}
\label{app:hparams}

\paragraph{Executor and judge.}
All methods share the same Executor, Claude Haiku~4.5 operating a tool loop with read, write, run, and validate actions inside the candidate workspace. The Executor is frozen throughout; it implements the plan it is given and returns the candidate together with execution and validation feedback. A per-batch judge, also Haiku~4.5, discards malformed hypotheses (empty, off-schema, or inexecutable plans) before execution in every arm; a discarded hypothesis is regenerated and does not consume a scored candidate.

\paragraph{Policies and adapters.}
\method uses Qwen3-4B-Instruct-2507 as the frozen backbone for both the Coordinator and the Scientists, served locally in bfloat16, with four independent LoRA adapters: one coordination adapter $\phi$ and one private Scientist adapter $\theta_i$ per unit ($N{=}3$). Each adapter has rank $r{=}32$, scaling $\alpha{=}64$, no dropout, and is applied to the seven linear projections of every transformer block (q, k, v, o, gate, up, down); each adapter has 66M parameters. The Coordinator emits the complete portfolio in one schema-constrained JSON completion of at most 1200 tokens containing, for each unit, a parent identifier chosen from a retrieved shortlist of at most 16 tree nodes (at most two nodes from the same family) and a free-text directive with three fields (research question, observed bottleneck, constraints). Each Scientist emits a hypothesis and an execution plan of at most 3400 tokens for its assignment. All adapters are reset to their initial checkpoint at the start of every task.

\paragraph{Adaptation.}
Adaptation follows Section~\ref{sec:coadaptation}. A group consists of $G{=}8$ forks of one frozen snapshot, each executing all $N{=}3$ units, hence 24 scored candidates per group; Coordinator groups and Scientist groups alternate strictly, and each group is followed by one optimizer step on its active block, after which the updated serving adapters are used for the next group. When the remaining scored-candidate budget is insufficient for another full matched-snapshot group, the remainder is used for inference without an additional update. The leave-one-out advantage of Eq.~(\ref{eq:grpo_objective}) is standardized by the within-group standard deviation, broadcast to all response tokens, and combined with the asymmetric clipping range and the KL penalty to the frozen task-start policy. Table~\ref{tab:hparams} lists the values.

\begin{table}[h]
\centering
\caption{Test-time adaptation hyperparameters for \method.}
\label{tab:hparams}
\small
\begin{tabular}{ll}
\toprule
Backbone & Qwen3-4B-Instruct-2507 (frozen, bfloat16) \\
Adapters & 1 Coordinator + 3 Scientist LoRA, $r{=}32$, $\alpha{=}64$, dropout 0 \\
Research units $N$ / forks per group $G$ & 3 / 8 \\
Candidates per group & $G{\times}N = 24$ \\
Optimizer steps per group & 1 \\
Clipping range $[1-\epsilon_{\mathrm{lo}},1+\epsilon_{\mathrm{hi}}]$ & $[0.8,\,1.28]$ \\
KL coefficient $\beta_{\mathrm{KL}}$ (to task-start policy) & $0.01$ \\
Learning rate / optimizer / gradient clip & $10^{-5}$ / AdamW / 1.0 \\
Sampling temperature / top-$p$ & $1.0$ / $1.0$ \\
Max new tokens, Coordinator / Scientist & 1200 / 3400 \\
Potential temperature $\tau$ & $0.05$ \\
Parent shortlist size / family cap & 16 / 2 \\
\bottomrule
\end{tabular}
\end{table}

\subsection{Baseline Implementations}
\label{app:baselines}

The three baselines are re-implemented in the same harness as \method while preserving the core search mechanisms of their cited methods. AIRA~\citep{toledo2026ai} formulates research as graph-based search over candidate solutions: it selects previously evaluated nodes, applies LLM-driven Draft, Debug, and Improve operators, and adds the evaluated artifact back to the search graph. MLEvolve~\citep{du2026mlevolve} uses Progressive MCGS, augmenting tree search with graph-based reference edges for cross-branch information flow, a progressive exploration-to-exploitation schedule, and retrospective memory for task-specific experience reuse. CORAL~\citep{coral} runs multiple autonomous research agents asynchronously; agents explore and reflect in parallel, communicate discoveries through shared persistent memory, and use periodic heartbeat-style interventions to sustain and redirect the search. All three use Haiku~4.5 as a frozen Scientist together with the shared Executor, evaluator, tree bookkeeping, and judge. We additionally run ARTS$^{*}$~\citep{arts}, the test-time-training variant of ARTS, in which a single reasoning Scientist inspects prior execution logs, diagnoses whether failures stem from faulty implementations or weak hypotheses, selects which tree node to expand, and is periodically fine-tuned with LoRA to distill the search tree into its weights. ARTS$^{*}$ uses the Qwen3-4B-Instruct-2507 backbone of \method{} together with the shared Executor, evaluator, tree bookkeeping, and judge.

\subsection{Ablation Protocol}
\label{app:ablation_protocol}

\paragraph{Variants and budget.}
We compare full \method{}, Coordinator-only, Scientist-only, and
Frozen on Iterated Prisoner's Dilemma (PD) and Colonel Blotto
(CB), using three independent seeds per variant.
The variants share the task-start policy initialization,
backbone, prompts, sampling settings, Executor, judge, and
tree bookkeeping. They differ only in which optimizer steps
are enabled: full \method{} updates both policy levels,
Coordinator-only updates the Coordinator, Scientist-only
updates the private Scientist adapters, and Frozen disables
both updates while retaining the same tree-based search.
Each run uses 72 candidate evaluations, comprising two
24-candidate adaptation groups and a final 24-candidate
evaluation batch reserved before the run. The reserved batch
is excluded from training. Timeouts and validation failures
consume budget in every variant. Model and optimization
settings follow Appendix~\ref{app:hparams}.

\paragraph{Adaptation schedule.}
The first group, denoted $b_{\mathrm{pre}}$, is collected from
the task-start snapshot before any parameter update.
Following the Coordinator-group protocol of Section~\ref{sec:coadaptation},
it consists of $G=8$ independent forks, each sampling a fresh
joint portfolio for $N=3$ research units and executing their
Scientists' responses. The snapshot and behavior policies
remain fixed during collection. Full \method{} and Coordinator-only
then take one optimizer step on the Coordinator adapter;
Scientist-only and Frozen skip this step. Validated candidates
are committed to the discovery tree in every variant.

The second group, denoted $b_{\mathrm{mid}}$, follows the
Scientist-group protocol. From the resulting snapshot, the
current Coordinator samples one portfolio, which is held fixed
across eight forks. Each Scientist independently samples its
response in each fork, producing another 24 candidate
evaluations. Full \method{} and Scientist-only then take one
optimizer step on each private Scientist adapter, with these
updates performed in parallel; Coordinator-only and Frozen
skip this step. Validated candidates are again committed to
the tree. After each enabled update, the updated serving
adapters are used for subsequent generation. Full \method{}
therefore completes one Coordinator--Scientist adaptation
cycle before the final batch is collected.

\paragraph{Post-cycle evaluation.}
After the second adaptation group has been processed and all
enabled updates applied, we generate a fresh 24-candidate
batch $b_{\mathrm{post}}$ using the resulting policies.
This batch consists of eight independent forks of the
post-adaptation snapshot. Each fork samples a new joint
portfolio and fresh responses from all three Scientists.
Thus, $b_{\mathrm{pre}}$ and $b_{\mathrm{post}}$ use the same
portfolio-resampling scheme, whereas $b_{\mathrm{mid}}$
holds one portfolio fixed for Scientist training.

All policy parameters, the discovery state, and the parent
shortlist remain fixed throughout collection of
$b_{\mathrm{post}}$. Candidates and feedback from one
evaluation fork are not exposed to the others, and no
optimizer step is taken on this batch. Each variant is
evaluated from the tree produced by its own adaptation
groups, so the comparison includes how the enabled updates
shape the subsequent discovery process.

\paragraph{Reported statistics.}
A valid candidate counts as a threshold hit when its raw
private-evaluator score (Appendix~\ref{app:integrity}) is at least $2.62$ on PD or
$0.245$ on CB, and the same thresholds apply to all variants and seeds.
For each seed and each reported batch, we
compute the median raw score over valid candidates and the
number of threshold hits among the 24 attempted
evaluations. Failed evaluations count as non-hits.

Table~\ref{tab:adaptation_ablation} reports
$b_{\mathrm{pre}}\!\rightarrow b_{\mathrm{post}}$,
averaging each statistic over the three seeds.
Batch medians are computed separately within each seed
before averaging, rather than by pooling candidates across
seeds. Both statistics use only candidates generated in the
corresponding batch, not candidates accumulated in the tree.

\begin{table}[t]
\centering
\small
\setlength{\tabcolsep}{5pt}
\caption{
Adaptation ablation on Prisoner's Dilemma (PD) and Colonel
Blotto (CB), averaged over three independent seeds.
Each entry reports
$b_{\mathrm{pre}}\!\rightarrow b_{\mathrm{post}}$:
the initial 24-candidate batch and a fresh 24-candidate batch
generated after all enabled updates in one scheduled
Coordinator--Scientist cycle.
``Median'' is the mean of the per-seed medians over valid
candidates; ``hits'' is the mean number of threshold hits
(score at least $2.62$ on PD or $0.245$ on CB) per 24 attempted evaluations.
}
\label{tab:adaptation_ablation}
\begin{tabular}{@{}lcccc@{}}
\toprule
& \multicolumn{2}{c}{PD}
& \multicolumn{2}{c}{CB} \\
\cmidrule(lr){2-3}
\cmidrule(lr){4-5}
Variant
& Median
& Hits / 24
& Median
& Hits / 24 \\
\midrule
Full \method{}
& $2.35 \rightarrow 2.57$
& $2 \rightarrow 10$
& $0.000 \rightarrow 0.010$
& $2 \rightarrow 6$ \\
Coordinator-only
& $2.35 \rightarrow 2.44$
& $2 \rightarrow 5$
& $0.000 \rightarrow 0.004$
& $2 \rightarrow 3$ \\
Scientist-only
& $2.35 \rightarrow 2.48$
& $2 \rightarrow 6$
& $0.000 \rightarrow 0.006$
& $2 \rightarrow 4$ \\
Frozen
& $2.35 \rightarrow 2.37$
& $2 \rightarrow 3$
& $0.000 \rightarrow 0.002$
& $2 \rightarrow 2$ \\
\bottomrule
\end{tabular}
\end{table}

\subsection{Aggregate Metrics}
\label{app:aggregate}

Figure~\ref{fig:aggregate} is computed from the three-seed mean per-task best scores of Table~\ref{tab:mlgym-main-results} and the anchors of Table~\ref{tab:tasks}. Scores are direction-corrected and normalized against the shipped baseline and human-best reference, $z=\operatorname{clip}_{[0,1]}\big((\tilde{s}-\tilde{s}_{0})/(\tilde{s}_{\mathrm{human}}-\tilde{s}_{0})\big)$, where $\tilde{s}$ increases with quality; Table~\ref{tab:normalized} lists the resulting values. Median, Mean, and IQM (the mean of the middle two quartiles) are taken over the 12 tasks; Optimality Gap is $\mathbb{E}[1-z]$, which under clipping equals $1-\mathrm{Mean}$. Confidence intervals are 95\% percentile intervals from a bootstrap that resamples the 12 task-level three-seed means with replacement (5{,}000 replicates, fixed bootstrap seed). Titanic clips to $z{=}1$ for every method, since all exceed the 0.830 reference.

\begin{table}[h]
\centering
\caption{Normalized, clipped scores $z$ behind Figure~\ref{fig:aggregate}.}
\label{tab:normalized}
\small
\begin{tabular}{lccccc}
\toprule
Task & MLEvolve & AIRA & CORAL & ARTS$^{*}$ & \method \\
\midrule
Titanic & 1.00 & 1.00 & 1.00 & 1.00 & 1.00 \\
House Prices & 0.45 & 0.45 & 0.40 & 0.46 & 0.57 \\
Fashion-MNIST & 0.74 & 0.82 & 0.79 & 0.81 & 0.91 \\
CIFAR-10 & 0.92 & 0.94 & 0.94 & 0.93 & 0.93 \\
MNLI & 0.79 & 0.79 & 0.79 & 0.79 & 0.79 \\
Language Modeling & 0.49 & 0.58 & 0.60 & 0.59 & 0.60 \\
Battle of the Sexes & 0.66 & 0.66 & 0.66 & 0.66 & 0.66 \\
Prisoner's Dilemma & 0.43 & 0.42 & 0.42 & 0.43 & 0.44 \\
Colonel Blotto & 0.67 & 0.67 & 0.67 & 0.67 & 0.67 \\
MountainCar & 0.31 & 0.38 & 0.61 & 0.67 & 0.86 \\
Breakout & 0.14 & 0.55 & 0.51 & 0.63 & 0.74 \\
MetaMaze & 0.48 & 0.65 & 0.69 & 0.69 & 0.76 \\
\bottomrule
\end{tabular}
\end{table}

\subsection{Protocol and Compute}
\label{app:protocol}

\paragraph{Budget accounting.}
Within each task, all methods receive the same nominal scored-candidate budget, so comparisons are matched by scored candidates rather than wall-clock time. Every candidate is executed under the task training timeout and scored by the private evaluator. A candidate that exceeds the timeout or fails validation consumes one evaluator call, receives failure utility $s_{\min}{=}0$ for the training potential, and is not added to the derivation tree.

\paragraph{Scores and seeds.}
Main benchmark results and adaptation ablations are averaged over three independent seeds. Table~\ref{tab:mlgym-main-results} reports the mean of each seed's best held-out score within budget, and Figure~\ref{fig:training-dynamics} aggregates the corresponding anytime best-so-far trajectories. Table~\ref{tab:adaptation_ablation} analogously averages the per-seed ablation statistics. Normalized scores use the anchors of Table~\ref{tab:tasks}; this reporting scale is distinct from the bounded training utility $s(c)=g_m(q_m(c))\in[0,1]$ used by the potential, which serves only task-time adaptation.

\paragraph{Compute.}
Experiments run on four NVIDIA A100 GPUs. Each method run takes approximately 9 hours, with test-time adaptation updating only the LoRA blocks described in Table~\ref{tab:hparams}.

\paragraph{Broader impact.}
More capable automated discovery systems can reduce the cost of exploring large hypothesis spaces and accelerate iterative scientific and machine-learning experimentation. The same autonomy can also propagate erroneous hypotheses or be applied in sensitive domains; practical deployments should therefore retain task-appropriate evaluation and human oversight.

\end{document}